\documentclass[sigconf]{acmart}

\usepackage{booktabs} 

\setcopyright{acmcopyright}

\copyrightyear{2019}
\acmYear{2019}
\setcopyright{acmcopyright}
\acmConference[WSDM '19]{The Twelfth ACM International Conference on Web Search and Data Mining}{February 11--15, 2019}{Melbourne, VIC, Australia}
\acmBooktitle{The Twelfth ACM International Conference on Web Search and Data Mining (WSDM '19), February 11--15, 2019, Melbourne, VIC, Australia}
\acmPrice{15.00}
\acmDOI{10.1145/3289600.3290957}
\acmISBN{978-1-4503-5940-5/19/02}

\usepackage{algorithm}
\usepackage{algorithmic}
\usepackage{tikz}
\usetikzlibrary{fit,positioning}
\usepackage{paralist}
\usepackage{appendix}
\usepackage{color}
\usepackage{booktabs}
\usepackage{multirow}
\usepackage{caption}
\usepackage{subcaption}

\newtheorem{theorem}{Theorem}[section]

\newtheorem{definition}{Definition}

\newtheorem{remark}[theorem]{Remark}

\newcommand{\EE}{\mathbb{E}}

\newcommand{\tr}{\textnormal{Tr}\,}

\newcommand{\diag}{\textnormal{diag}}
\newcommand{\Diag}{\textnormal{Diag}\,}

\,
\,

\newcommand{\argmin}{\mathop{\rm argmin}}

\newcommand{\XCal}{\mathcal{X}}

\newcommand{\diam}{\textnormal{diam}}
\newcommand{\dom}{\text{dom}}

\newcommand{\br}{\mathbb{R}}

\newcommand{\ba}{\begin{array}}
\newcommand{\ea}{\end{array}}

\newcommand{\bp}{\mathbb{P}}
\newcommand{\bq}{\mathbb{Q}}

\newcommand{\etal}{{\it et al.\ }}

\title{Sparsemax and Relaxed Wasserstein for Topic Sparsity}

 \author{Tianyi Lin}
 \affiliation{
 \institution{University of California, Berkeley}
 \city{Berkeley}
 \state{California}
 }
 \email{darren\_lin@berkeley.edu}

 \author{Zhiyue Hu}
 \affiliation{
 \institution{University of California, Berkeley}
 \city{Berkeley}
 \state{California}
 }
 \email{zyhu95@berkeley.edu}

 \author{Xin Guo}
 \affiliation{
 \institution{University of California, Berkeley}
 \city{Berkeley}
 \state{California}
 }
 \email{xinguo@berkeley.edu}

\begin{document}

\begin{abstract}
Topic sparsity refers to the observation that individual documents usually focus on several salient topics instead of covering a wide variety of topics, and a real topic adopts a narrow range of terms instead of a wide coverage of the vocabulary. Understanding this topic sparsity is especially important for analyzing user-generated web content and social media, which are featured in the form of extremely short posts and discussions. As topic sparsity of individual documents in online social media increases, so does the difficulty of analyzing the online text sources using traditional methods.

In this paper, we propose two novel neural models by providing sparse posterior distributions over topics based on the Gaussian sparsemax construction, enabling efficient training by stochastic backpropagation. We construct an inference network conditioned on the input data and infer the variational distribution with the relaxed Wasserstein (RW) divergence. Unlike existing works based on Gaussian softmax construction and Kullback-Leibler (KL) divergence, our approaches can identify latent topic sparsity with training stability, predictive performance, and topic coherence. Experiments on different genres of large text corpora have demonstrated the effectiveness of our models as they outperform both probabilistic and neural methods. 
\end{abstract}


\terms{Algorithms, Experimentation, Performance}


\maketitle

\section{Introduction}
Social networks have become integral components of the web. According to Cisco Systems, the number of active websites surpassed one billion in 2016, up from approximately 700 million in 2012\footnote{http://en.wikipedia.org/wiki/user-generated content}. In a typical social network platform such as Twitter, the micro-blogging service is averaged at 335 million monthly active users in 2018, more than twice as many as in 2012\footnote{https://www.statista.com/statistics/282087/number-of-monthly-active-twitter-users/}. The huge amount of user-generated content, normally in the form of very short text, contains rich information that is barely found in traditional text sources yet is important for social media event detection, sentiment analysis, personalized recommendation, among others. Therefore, analyzing large-scale user-generated content in social media has been an emerging research direction.

One of the main challenges is to understand the topic sparsity in short text: different from carefully-edited articles, user-generated content in social media is extremely short with a very large vocabulary and a broad range of topics  \cite{Hong-2010-Empirical, Zhao-2011-Comparing}. Consequently, probabilistic topic models \cite{Hofmann-1999-Probabilistic, Blei-2003-Latent}  have experienced mixed results, despite their broad success on traditional media. Recent effort on sparsity-enhanced topic models yields limited success due to the complicated procedure to infer topic sparsity on large-scale text corpora \cite{Shashanka-2008-Sparse, Wang-2009-Decoupling, Zhu-2011-Sparse, Eisenstein-2011-Sparse, Chen-2012-contextual, Lin-2014-Dual, Xu-2014-Latent, Lin-2016-Understanding}. The latest development on topic modeling is to incorporate the deep neural networks with either the generative process \cite{Hinton-2009-Replicated, Larochelle-2012-Neural, Cao-2015-Novel, Card-2017-Neural, Peng-2018-Neural} or the inference method \cite{Kingma-2014-Auto, Rezende-2014-Stochastic, Mnih-2014-Neural, Miao-2016-Neural, Miao-2017-Discovering, Srivastava-2017-Autoencoding}. Compared to traditional inference methods \cite{Jordan-1999-Introduction, Hoffman-2013-Stochastic}, this approach is more efficient and more accurate with the training based on backpropagation; it is also more adaptive to infer new models given a simple declarative specification of the generative process. 
However, all existing neural approaches are based on the Kullback-Leibler (KL) divergence which is not suitable for inferring topic sparsity. Indeed, as the true distribution is sparse, or in other words, supported on a low dimensional manifold, KL divergence has shown to be \textit{unsuitable} and contributing to the instability of training \cite{Arjovsky-2017-Towards}. 

In this paper, we propose two new neural models, namely Neural SparseMax Document and Topic Models (NSMDM and NSMTM), which apply the ``sparsemax" model of attention \cite{Martins-2016-Softmax} to induce the topic sparsity. To efficiently infer the topic sparsity from large-scale text corpora, we design a new neural variational inference framework based on the relaxed Wasserstein (RW) divergence \cite{Guo-2017-Relaxed}. The proposed approach is shown to outperform all existing methods in terms of the quality of reconstruction while maintaining the stability of training. Moreover, the training and testing is much faster than traditional methods on large-scale text corpora. 

To the best of our knowledge, these are the first deep neural document and topic models that efficiently identify topic sparsity from online social media. Experiments on different genres of large-scale text corpora demonstrate that NSMDM and NSMTM address sparsity in both document-topic and topic-word structure of text corpora, and consistently outperform other competing methods on large-scale short text corpora, in terms of training stability, predictive performance, and topic coherence.

The rest of the paper is organized as follows. Section~\ref{sec:related-work} lists several related work and discusses their relationships with our models, Section~\ref{sec:definition}
defines the problem of modeling topic sparsity in text corpora, Section~\ref{sec:model} introduces the Neural SparseMax Document and Topic Model (NSMDM and NSMTM), and the inference framework based on the RW divergence,  Section~\ref{sec:experiment} describes the experiments on different genres of large-scale short text corpora,
and Section~\ref{sec:conclusion} concludes.

\section{Related Work}\label{sec:related-work}

\subsection{Probabilistic Topic Models}
Probabilistic topic models have been one of the most successful approaches for unsupervised learnings. Without utilizing auxiliary information such as higher-level context, these models generate each document from a mixture of topics where each topic is defined as a unigram distribution over all the terms in the vocabulary. While classical topic models, such as probabilistic latent semantic analysis (PLSA) \cite{Hofmann-1999-Probabilistic} and latent Dirichlet allocation (LDA) \cite{Blei-2003-Latent} have enjoyed broad success on traditional media texts, their success on social media texts is limited. This limitation inspires a line of works on sparsity-enhanced topic models that address the problem of sparsity in document-topic and topic-term distributions. While some of these models apply the non-negative matrix factorization \cite{Hoyer-2004-Matrix} and topical coding \cite{Zhu-2011-Sparse, Zhang-2013-Sparse} with $\ell_1$-regularization to induce sparse posterior distribution, the result on tweets is still mixed \cite{Lin-2014-Dual}. Another category of sparsity-enhanced models improves classical models by adopting specific prior, such as an entropic prior \cite{Shashanka-2008-Sparse}, a spike and slab prior \cite{Wang-2009-Decoupling, Lin-2014-Dual}, and a zero-mean Laplace prior \cite{Eisenstein-2011-Sparse}, to decouple across-data prevalence and within-data proportion in modeling mixed membership data.  These models enjoy both effective structures and efficient inference from exploiting conjugacy with either Monte Carlo or mean-field variational techniques. However, as the expressiveness of these topic models grows, inference methods turn out to be increasingly complicated and intractable on large text corpora. 

\subsection{Neural Topic Models}
Deep neural networks have shown great potential for approximating complicated nonlinear distributions in unsupervised models. The resulting neural models can be efficiently trained by backpropagation \cite{Rezende-2014-Stochastic} while keeping the excellent probabilistic interpretation and the explicit dependence among latent variables. One of the representative categories is the neural document models, such as replicated softmax \cite{Hinton-2009-Replicated}, neural auto-regressive model \cite{Larochelle-2012-Neural}, belief networks \cite{Mnih-2014-Neural}, and neural variational document model \cite{Miao-2016-Neural}. However, these models do not explicitly model latent topics.  

The neural topic models \cite{Cao-2015-Novel, Miao-2017-Discovering, Card-2017-Neural}, on the other hand, directly extend the classical statistical topic models by replacing the \textit{Dirichlet-multinomial} construction in LDA with the Gaussian softmax construction, and significantly improve the expressiveness on large text corpora. However, these models are not able to produce sparse posterior distribution and probabilistic representations of topics, thus fail to address the skewness of the topic mixtures and the word distributions. Peng \etal \cite{Peng-2018-Neural} thus propose a neural sparse topic coding model and show that their approach outperforms sparse topical coding \cite{Zhu-2011-Sparse}. However, the improvement is not significant possibly because the probabilistic representation of topics is lost.

\subsection{Variational Inference}
The basic idea behind the variational inference framework is to learn the posterior distribution by optimizing the divergence between this distribution and a variational distribution. Standard methods for topic models contain \textit{mean-field variational inference} \cite{Hoffman-2013-Stochastic} and \textit{sampling-based varational inference} \cite{Kingma-2014-Auto, Mescheder-2017-Adversarial, Srivastava-2017-Autoencoding}. While the former is model specific and further assumes the conditional independence of latent variables, the latter  only requires very limited and easy-to-compute information from the model and thus is flexible for a variety of models \cite{Srivastava-2017-Autoencoding}.

All the existing inference frameworks in neural topic models are based on the KL divergence, which has shown to be \textit{unsuitable} and contributing to the instability of training \cite{Arjovsky-2017-Towards}. In contrast, the Wasserstein divergence \cite{Villani-2008-optimal} provides a meaningful and smooth representation of the distance in-between even when the true distribution is sparse, yielding a robust training in Generative Adversarial Network (GAN) \cite{Arjovsky-2017-Wasserstein} and Auto Encoder (AE) \cite{Tolstikhin-2018-Wasserstein}. Meanwhile, the RW divergence \cite{Guo-2017-Relaxed}, incorporating the Bregman function into the Wasserstein divergence, speeds up the training of Wasserstein GAN while keeping the stability and robustness.

 \section{Problem Definition}\label{sec:definition}
In this section, we define the problem of modeling topic sparsity. Let $D = \{\vec{w}_j\}_{j=1}^{|D|}$ be a text corpora where $\vec{w}_j=(w_{j1}, \ldots, w_{jn_{j}})$ is a vector of terms representing the textual content of document $j$. Here $w_{ji}$ refers to the frequency of term $i$ in document $j$ and $V$ refers to the vocabulary of distinct words in $D$.
\begin{definition}[Topic, Topical Structure, Topic Modeling] \label{Def:Topic}
A \emph{topic} $\vec{\phi}$ in a document collection $D$ is defined as a multinomial distribution over the vocabulary $V$ such that 
\[\bp(v=i \mid \vec{\phi}) = \phi_i, \qquad i=1,2,\ldots,\left|V\right|, \]
where $\left|V\right|$ denotes the size of the vocabulary. 

Similarly, the \emph{topical structure} $\vec{\theta}$ of a document is  defined as a multinomial distribution over $K$ topics such that 
\[\bp(\vec{\phi}=\vec{\phi}_k \mid \vec{\theta}) = \theta_k, \qquad k=1,2,\ldots,K,\]
where $K$ is the total number of topics contained in $D$, 

Given  a text corpus $D$,  \emph{topic modeling} aims to learn a set of salient topics and the topical structure of all documents,  $\{\vec{\phi}_k\}_{k=1}^K$ and $\{\vec{\theta}_j\}_{j=1}^{|D|}$.

\end{definition}

\begin{definition}[Topic Sparsity] \label{Def:Topic_Sparsity}
Topic sparsity means that individual documents usually focus on several salient topics instead of covering a wide variety of topics, and a real topic also adopts a narrow range of terms instead of a wide coverage of the vocabulary. That is, 
\begin{eqnarray*}
& 1 \leq \sum_{k=1}^K \text{1}_{(\theta_{jk}>0)} \ll K, & j=1,2,\ldots,\left|D\right|, \\
& 1 \leq \sum_{i=1}^{|V|} \text{1}_{(\phi_{ki}>0)} \ll |V|, & k=1,2,\ldots,K.
\end{eqnarray*}

\end{definition}
Most Bayesian topic models, such as LDA \cite{Blei-2003-Latent}, adopt the Dirichlet prior for both topics and the topic structure of documents. That is, 
\begin{eqnarray*}
& \vec{\theta}_j \sim \text{Dirichlet}\left(\vec{\alpha}\right), & j=1,\cdots, \left|D\right|, \\
& \vec{\phi}_k \sim \text{Dirichlet}\left(\vec{\beta}\right), & k=1,\cdots, K. 
\end{eqnarray*}
The Dirichlet prior alleviates the overfitting problem of PLSA \cite{Hofmann-1999-Probabilistic} in practice by smoothing the topic mixture in individual documents and the word distribution of each topic. Neural topic models, such as GSM \cite{Miao-2017-Discovering}, adopt the Gaussian softmax construction for both topics and the topic structure of documents, i.e., 
\begin{eqnarray*}
\vec{x} \sim \text{Gaussian}\left(0, I_d\right), & \vec{\theta}_j = \text{softmax}\left(W^\top\vec{x}\right), & j=1,\cdots, \left|D\right|, \\
& \vec{\phi}_k =\text{softmax}\left(S^\top\vec{t}_k\right), & k=1,\cdots, K. 
\end{eqnarray*}
The Gaussian softmax construction is simple to evaluate and differentiate, enabling the efficient implementation of stochastic backpropagation \cite{Martins-2016-Softmax}. However, neither the Dirichlet prior nor the Gaussian softmax construction is suitable for modeling topic sparsity (Definition~\ref{Def:Topic_Sparsity}) since they do not formally control the posterior sparsity of the inferred topical structure as discussed earlier.

Given a collection of documents $D$, the vocabulary $V$ and the number of topics $K$, the major task of topic sparsity modeling can be defined as
\begin{enumerate} 
\item inferring the sparse topic proportion of document $j$, i.e., $\vec{\theta}_j$;
\item inferring the sparse word usage of topic $k$, i.e., $\vec{\phi}_k$.
\end{enumerate}
All the notations used in this paper are listed in Table~\ref{Tab:Notation}.
\begin{table}[t]
\centering
\caption{Variables in Neural Topic Modeling} \label{Tab:Notation}
\vspace{-1em}
\begin{tabular}{|r|l|} \hline
\textbf{Notation}& \textbf{Definition}\\ \hline
$K$ & number of topics \\ \hline
$V$ & vocabulary\\ \hline
$D$ & a collection of documents\\ \hline
$N_j$ & length of document $j$\\ \hline
$w_{ji}$ & word $i$ in document $j$\\ \hline
$\vec{w}$ & a set of all words, i.e., $\left\{\vec{w}_j\right\}_{j=1}^{|D|}$\\ \hline
$z_{ji}$ & assigned topic at $i$th word in document $d$\\ \hline
$\vec{z}$ & a set of all topic assignments, i.e., $\{\vec{z}_{j}\}_{j=1}^{|D|}$ \\ \hline
$\vec{\theta}_j$ & topical structure of document $j$\\ \hline
$\vec{\phi}_{k}$ & word usage of topic $k$\\ \hline
$\vec{\phi}$ & a dictionary matrix $\in \br^{K\times V}$\\ \hline
$\vec{\psi}_j$ & word structure of document $j$\\ \hline
$\left(\mu_0, \sigma_0^2\right)$ & hyper-parameters for the Gaussian prior \\ \hline
$\gamma$ & regularization parameter \\ \hline
$\bp$, $\bq$ & probability distributions \\ \hline
$\left(\XCal, \Sigma\right)$ & a measurable space \\ \hline
$\text{Gaussian}(\cdot)$ & Gaussian distribution\\ \hline
$\text{Multinomial}(\cdot)$ & Multinomial distribution\\ \hline
$\diam(\cdot)$ & a diameter of a set \\ \hline
$\dom(\cdot)$ & a domain of a function \\ \hline
$\text{1}(\cdot)$ & Indicator function \\ \hline
$\left\|\cdot\right\|$ & $\ell_2$ norm \\ \hline 
$\tr(\cdot)$ & the trace of a matrix. \\ \hline
$\log(\cdot)$ & the natural logarithm. \\ \hline 
\end{tabular} \vspace{-2em}
\end{table}

\section{Methodology}\label{sec:model}
Topic sparsity is the common observation in online social media, such as Twitter and Facebook. It is challenging for the recently proposed neural topic models in identifying the sparse structure of documents and topics. To address this problem, we propose to induce sparsity by replacing the Gaussian softmax construction by the Gaussian sparsemax construction in the generative network. More specifically, we introduce two new neural models, Neural SparseMax Document and Topic Models (NSMDM and NSMTM), where the generative process is inspired by the \textit{sparsemax} model of attention \cite{Martins-2016-Softmax}. Meanwhile, to make the inference network work, we use the RW divergence to approximate the posterior by the variational distribution. Combined, our approaches model sparse document-topic and topic-term distributions effectively and infer this sparsity from large-scale text corpora efficiently. 

\subsection{Generative Network}
We describe the generative process of $\vec{\theta}$ and $\vec{\phi}$ in our NSMDM and NSMTM models.  $\vec{\theta}$ and $\vec{\phi}$ are both generated from the Gaussian sparsemax construction. As a result, $\vec{\theta}$ and $\vec{\phi}$ are sparse since the projection is likely to hit the boundary of the simplex. 

\noindent\textbf{NSMDM:} The model is depicted in Figure~\ref{Figure:NSMDM} and its generative process is presented as follows: \\

\noindent For each topic indexed by $k \in \{1,2,\ldots,K\}$:
\begin{enumerate}
\item the topic distribution $\vec{\phi}_{k}=S^\top\vec{t}_k$.
\end{enumerate}
For document indexed by $j\in\{1,2,\ldots,|D|\}$:
\begin{enumerate}
\item $\vec{x}_j\sim$ Gaussian$(\mu_0, \sigma_0^2)$;
\item the topic proportion $\vec{\theta}_j=\text{sparsemax}(W^\top\vec{x}_j)$;
\item the word distribution $\vec{\psi}_j = \text{softmax}(\vec{\phi}^\top\vec{\theta}_j)$; 
\item For each word indexed by $i \in \{1,2,\ldots,N_j\}$:
\begin{enumerate}
\item sample $w_{ji}$ from Multinomial$\left(\vec{\psi}_j\right)$. 
\end{enumerate}
\end{enumerate}
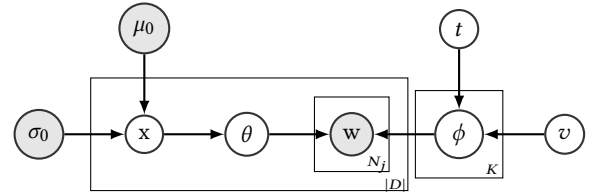
\begin{figure}[!ht]
\centering
\begin{tikzpicture}
\tikzstyle{main}=[circle, minimum size = 0.1mm, thick, draw =black!80, node distance = 8mm]
\tikzstyle{connect}=[-latex, thick]
\tikzstyle{box}=[rectangle, draw=black!100]
\node[main, fill = white!100] (x) {x};
\node[main,fill = black!10] (mu) [above =of x] {$\mu_0$};
\node[main,fill = black!10] (sigma) [left=of x] {$\sigma_0$};
\node[main] (theta) [right=of x]{$\theta$};
\node[main, fill = black!10] (w) [right=of theta]{w};
\node[main] (phi) [right=of w]{$\phi$};
\node[main] (v) [right=of phi]{$v$};
\node[main] (t) [above =of phi] {$t$};
\path (mu) edge [connect] (x)
         (sigma) edge [connect] (x)
         (x) edge [connect] (theta)
         (theta) edge [connect] (w)
         (v) edge [connect] (phi)
         (t) edge [connect] (phi)
         (phi) edge [connect] (w);
\node[rectangle, inner sep=-1.2mm, fit= (w), label=below right:\tiny $N_j$, xshift=-1mm] {};
\node[rectangle, inner sep=2mm, draw=black!100, fit= (w)] {};
\node[rectangle, inner sep=1.7mm, fit=(w), label=below right: \tiny $|D|$, xshift=-1.5mm] {};
\node[rectangle, inner sep=4.5mm, draw=black!100, fit = (x)(theta) (w)] {};
\node[rectangle, inner sep=-0.5mm, fit= (phi), label=below right:\tiny $K$, xshift=-0.5mm] {};
\node[rectangle, inner sep=2.5mm, draw=black!100, fit= (phi)] {};
\end{tikzpicture}
\vspace{-.5em} \caption{The generative process of NSMDM} \label{Figure:NSMDM} \vspace{-.5em}
\end{figure}
\textbf{NSMTM:} The model is depicted in Figure~\ref{Figure:NSMTM} and the generative network is presented as follows: \\

\noindent For each topic indexed by $k \in \{1,2,\ldots,K\}$:
\begin{enumerate}
\item the topic distribution $\vec{\phi}_{k}=\text{sparsemax}(S^\top\vec{t}_k)$.
\end{enumerate}
For document indexed by $j\in\{1,2,\ldots,|D|\}$:
\begin{enumerate}
\item $\vec{x}_j\sim$ Gaussian$(\mu_0, \sigma_0^2)$;
\item the topic proportion $\vec{\theta}_j=\text{sparsemax}(W^\top\vec{x}_j)$;
\item For each word indexed by $i \in \{1,2,\ldots,N_j\}$:
\begin{enumerate}
\item sample $z_{ji}$ from Multinomial$(\vec{\theta}_j)$;
\item sample $w_{ji}$ from Multinomial$(\vec{\phi}_{z_{ji}})$. 
\end{enumerate}
\end{enumerate}
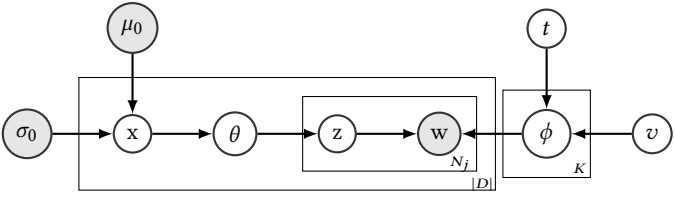
\begin{figure}[!ht]
\centering
\begin{tikzpicture}
\tikzstyle{main}=[circle, minimum size = 0.1mm, thick, draw =black!80, node distance = 8mm]
\tikzstyle{connect}=[-latex, thick]
\tikzstyle{box}=[rectangle, draw=black!100]
\node[main, fill = white!100] (x) {x};
\node[main,fill = black!10] (mu) [above =of x] {$\mu_0$};
\node[main,fill = black!10] (sigma) [left=of x] {$\sigma_0$};
\node[main] (theta) [right=of x]{$\theta$};
\node[main] (z) [right=of theta]{z};
\node[main, fill = black!10] (w) [right=of z]{w};
\node[main] (phi) [right=of w]{$\phi$};
\node[main] (v) [right=of phi]{$v$};
\node[main] (t) [above =of phi] {$t$};
\path (mu) edge [connect] (x)
         (sigma) edge [connect] (x)
         (x) edge [connect] (theta)
         (theta) edge [connect] (z)
         (z) edge [connect] (w)
         (v) edge [connect] (phi)
         (t) edge [connect] (phi)
         (phi) edge [connect] (w);
\node[rectangle, inner sep=-1.2mm, fit= (z) (w),label=below right:\tiny $N_j$, xshift=5mm] {};
\node[rectangle, inner sep=2mm, draw=black!100, fit= (z) (w)] {};
\node[rectangle, inner sep=1.7mm, fit=(z) (w), label=below right: \tiny $|D|$, xshift=5mm] {};
\node[rectangle, inner sep=4.5mm, draw=black!100, fit = (x)(theta)(z) (w)] {};
\node[rectangle, inner sep=-0.5mm, fit= (phi), label=below right:\tiny $K$, xshift=-0.5mm] {};
\node[rectangle, inner sep=2.5mm, draw=black!100, fit= (phi)] {};
\end{tikzpicture}
\vspace{-.5em} \caption{The generative process of NSMTM} \label{Figure:NSMTM} \vspace{-.5em}
\end{figure}
We make the following comments on the \textit{sparsemax} construction.
\begin{itemize}
\item \textbf{Idea.}  It is necessary to understand the rationale behind the \textit{sparsemax} construction. Previous work \cite{Miao-2017-Discovering} has found it reasonable to use the Gaussian softmax construction to define both document-topic and topic-term distributions. However, the Gaussian softmax construction only induces the sparsity when some of the input vectors approach infinity. Specifically, a softmax function is defined as
\[ \left[\text{softmax}(\vec{x})\right]_j := \frac{e^{-x_j}}{\sum_{j=1}^n e^{-x_j}}, \]
implying that $\left[\text{softmax}(\vec{x})\right]_j\approx 0$ when $x_j$ tends to infinity. In contrast, Gaussian sparsemax construction can produce sparse probability distribution, given by
\begin{equation}\label{Def:Sparsemax}
\text{sparsemax}(\vec{x}) := \argmin_{\vec{p}\in\Delta^{d-1}} \ \|\vec{p}-\vec{x}\|_2^2, 
\end{equation}
where $\Delta^{d-1} := \left\{\vec{p}\in\br^d \mid \sum_{j=1}^d p_j = 1, \ \vec{p} \geq 0\right\}$. 
\item \textbf{Construction.} The sparsemax construction is simple to evaluate while keeping most of the appealing properties of the softmax construction \cite{Martins-2016-Softmax}. In fact, the solution to~\eqref{Def:Sparsemax} is of the form: 
\begin{equation*}
\left[\text{sparsemax}(\vec{x})\right]_j = \max\left\{0, x_j - \tau(\vec{x})\right\}, 
\end{equation*} 
where $\tau:\br^d\rightarrow\br$ is the unique function so that the sum of all $\left[\text{sparsemax}(\vec{x})\right]_j$ is 1 for any $\vec{x}\in\br^d$. More specifically, let $x_{(1)} \geq x_{(2)} \geq \ldots \geq x_{(d)}$ be the sorted coordinates of $\vec{x}$ and $T(\vec{x})$ be the maximum number of $k$ that $1+kx_{(k)}>\sum_{j\leq k} x_{(j)}$, then 
\[ \tau(\vec{x}) = \frac{\sum_{j \leq T(\vec{x})} x_{(j)}-1}{T(\vec{x})} = \frac{\sum_{j \in S(\vec{x})} x_{(j)}-1}{S(\vec{x})}, \]
where $S(\vec{x})$ is the support of $\text{sparsemax}(\vec{x})$, i.e., a set of the indices of nonzero coordinates. Finally, the sparsemax construction is easy to differentiate, with the Jacobian matrix given by
\[ \text{Jacobian}(\vec{x}) = \Diag(s) - \frac{ss^\top}{T(\vec{x})}, \] 
where $s$ is an indicator vector whose $i$th entry is 1 if $i\in S(\vec{x})$ and 0 otherwise. 
\end{itemize}

\subsection{Inference Framework}\label{subsec:inference-network}
In this subsection, we develop a new neural inference method based on the RW divergence.
In addition to the reparameterization tricks \cite{Kingma-2014-Auto} for an unbiased gradient estimation with 
a low variance, we carry out the inference with the RW regularization between variational distribution and the prior.

\textbf{Variational Bound:} For NSMDM, first recall a variational lower bound for the document log-likelihood 
\begin{eqnarray*}
\log\left(p\left(\vec{w} \mid \mu_0, \sigma_0, \vec{\phi}\right)\right) & = &  \log\left(\int_{\vec{\theta}} p\left(\vec{\theta} \mid \mu_0, \sigma_0^2\right) \prod_{j=1}^{|D|} \prod_{i=1}^{N_j} p\left(\vec{w}_{ji} \mid \vec{\phi}, \vec{\theta}_j\right) d\vec{\theta}\right) \nonumber \\
& \geq & \EE_{q(\vec{\theta}\mid \vec{w})}\left[\sum_{j=1}^{|D|}\sum_{i=1}^{N_j} \log\left(p\left(\vec{w}_{ji} \mid \vec{\phi}, \vec{\theta}_j\right)\right)\right] \\
& & - D_{\textsf{KL}}\left[q(\vec{\theta}\mid \vec{w}) \parallel p(\vec{\theta} \mid \mu_0, \sigma_0^2)\right], 
\end{eqnarray*}
where $q(\vec{\theta}\mid \vec{w})$ is the variational distribution approximating the true posterior $p(\vec{\theta}\mid \vec{w})$ and the prior distribution is defined in which $\vec{x}\sim\text{Gaussian}(\mu_0, \sigma_0^2)$ and $\vec{\theta}_j=\text{sparsemax}(W^\top\vec{x}_j)$. The second term is the KL regularization which forces $q(\vec{\theta}\mid \vec{w})$ to be close to $p(\theta\mid\mu_0, \sigma_0^2)$. However, it may result in an unstable training since this term is likely to be infinity if $q(\vec{\theta}\mid \vec{w})$ and $p(\theta\mid\mu_0, \sigma_0^2)$ are supported on different low dimensional manifolds and is therefore not suitable for mining the topic sparsity. In contrast, the RW divergence, the generalization of Wasserstein divergence, can avoid the above issues in KL divergence. We refer the interested reader to \cite{Guo-2017-Relaxed} for more details. 
\begin{definition}[Relaxed Wasserstein divergence]\label{Def:RWD}
The RW divergence between $\bp$ and $\bq$ on $\left(\XCal, \Sigma\right)$ is defined as
\[
W_{D_\varphi} \left(\bp, \bq\right) = \inf_{\pi\in\prod(\bp, \bq)} \int_{\XCal\times\XCal} \ D_\varphi\left(\vec{x}, \vec{y}\right) \ \pi\left(d\vec{x}, d\vec{y}\right), 
\]
where $\prod(\bp, \bq)$ a set of probability distributions with marginal distributions $\bp$ and $\bq$, and 
\begin{equation*}
D_\varphi\left(\vec{x}, \vec{y}\right) = \varphi(\vec{x}) - \varphi(\vec{y}) - (\vec{x}-\vec{y})^\top\nabla\varphi(\vec{x}),
\end{equation*}
and $\varphi$ is a strictly convex function with the $L_\varphi$-Lipschitz continuous gradient, i.e., $\left\|\nabla\varphi(\vec{x}) - \nabla\varphi(\vec{y})\right\|\leq L_\varphi\left\|\vec{x}-\vec{y}\right\|$ for $\vec{x}, \vec{y}\in\dom(\varphi)$. 
\end{definition}
Now, we can derive a new variational bound as 
\begin{eqnarray*}
L_{\textsf{NSMDM}} & = & \EE_{q(\vec{\theta}\mid \vec{w})}\left[\sum_{j=1}^{|D|}\sum_{i=1}^{N_j} \log\left(p\left(\vec{w}_{ji} \mid \vec{\phi}, \vec{\theta}_j\right)\right)\right] \\
& & - \gamma \cdot W_{D_\varphi}\left[q(\vec{\theta}\mid \vec{w}) \parallel p(\vec{\theta} \mid \mu_0, \sigma_0^2)\right] \\
& \approx & \sum_{j=1}^{|D|}\sum_{i=1}^{N_j} \log\left(\text{softmax}(\vec{\phi}^\top\hat{\vec{\theta}}_j)_{w_{ji}}\right) \\
& & - \gamma \cdot W_{D_\varphi}\left[q(\vec{\theta}\mid \vec{w}) \parallel p(\vec{\theta} \mid \mu_0, \sigma_0^2)\right], \quad  \hat{\vec{\theta}}_j \sim q(\vec{\theta}\mid \vec{w}). 
\end{eqnarray*}
Similarly, a new variational lower bound for NSMTM is as follows, 
\begin{eqnarray*}
L_{\textsf{NSMTM}} & = & \EE_{q(\vec{\theta}\mid \vec{w})}\left[\sum_{j=1}^{|D|}\sum_{i=1}^{N_j} \log\left( \sum_{z_{ji}} p\left(\vec{w}_{ji} \mid \vec{\phi}_{z_{ji}}\right) p\left(z_{ji} \mid \vec{\theta}_j\right)\right)\right] \\
& & - \gamma \cdot W_{D_\varphi}\left[q(\vec{\theta}\mid \vec{w}) \parallel p(\vec{\theta} \mid \mu_0, \sigma_0^2)\right] \\
& \approx & \sum_{j=1}^{|D|}\sum_{i=1}^{N_j} \log\left((\vec{\phi}^\top\hat{\vec{\theta}}_j)_{w_{ji}}\right) \\
& & - \gamma \cdot W_{D_\varphi}\left[q(\vec{\theta}\mid \vec{w}) \parallel p(\vec{\theta} \mid \mu_0, \sigma_0^2)\right], \quad \hat{\vec{\theta}}_j \sim q(\vec{\theta}\mid \vec{w}). 
\end{eqnarray*}
Generally speaking, the new variational bound equals to the document log-likelihood when $q(\vec{\theta}\mid \vec{w})=p(\vec{\theta} \mid \mu_0, \sigma_0^2)$ but may not be necessarily smaller if $\gamma=1$. So it is unclear (yet) if this new bound can be a reasonable objective for some proper choices of $\gamma$ in our optimization framework. Fortunately, Theorem~\ref{theorem:Wass-KL} below provides a positive answer by specifying the relationship between the new bound and the original variational bound based on KL divergence. 
\begin{theorem}\label{theorem:Wass-KL}
Given two probability distributions $\bp$ and $\bq$ on $\left(\XCal, \Sigma\right)$,  we have
\begin{equation*}
\frac{1}{L_\varphi\left[\diam(\XCal)\right]^2}W_{D_\varphi}\left(\bp, \bq\right) \leq TV\left(\bp, \bq\right) \leq \sqrt{\frac{1}{2} D_{\textsf{KL}}\left(\bp \parallel \bq\right)},
\end{equation*}
where $L_\varphi>0$ is defined in Definition~\ref{Def:RWD} and the total variation distance is defined as
\[ TV\left(\bp, \bq\right) = \sup_{A\in\Sigma} \ \left|\bp(A) - \bq(A)\right|. \]
\end{theorem}
\begin{proof}
The first inequality comes from Theorem~3.1 \cite{Guo-2017-Relaxed} and the second inequality is the restatement of \textit{Pinsker's inequality} \cite{Cover-2012-Elements}. 
\end{proof}
\textbf{RW Regularization:} Given the generative distribution $p(\vec{\theta}\mid\mu_0, \sigma_0^2)=p(\vec{x}\mid\mu_0, \sigma_0^2)$ and the variational distribution $q(\vec{\theta}\mid \vec{w}) = q(\vec{x}\mid \vec{\mu}(\vec{w}), \vec{\sigma}(\vec{w}))$, the RW term can be easily integrated analytically for $\varphi(\cdot)=\left\|\cdot\right\|^2$ and the Gaussian distributions\cite{Dowson-1982-Frechet, Knott-1984-Optimal}, where the closed-form solution is summarized in the following theorem.  
\begin{theorem}
Assume that $\varphi(\cdot)=\left\|\cdot\right\|^2$, $\bp=\text{Gaussian}\left(\vec{\mu}_p, \Sigma_p\right)$, and $\bq=\text{Gaussian}\left(\vec{\mu}_q, \Sigma_q\right)$, then
\[
W_{D_\varphi} \left(\bp, \bq\right) = \left\|\vec{\mu}_p - \vec{\mu}_q\right\|^2 + \tr\left(\Sigma_p + \Sigma_q - 2\left(\Sigma_p\Sigma_q\right)^{1/2}\right). 
\]
\end{theorem}
\textbf{Inference Network $q(\vec{x}\mid \vec{w})$:} The inference network is constructed as follows: 
\[ \vec{x} \sim \text{Gaussian}\left(\vec{\mu}(\vec{w}), \diag\left(\vec{\sigma}^2(\vec{w})\right)\right) \],
where
\begin{eqnarray*}
& \vec{\lambda}_1 = \textsf{ReLU}\left(W_1\vec{w} + \vec{b}_1\right), & \vec{\lambda}_2 = \textsf{ReLU}\left(W_2\vec{\lambda}_1 + \vec{b}_2\right), \\
& \vec{\mu}(\vec{w}) = W_3\vec{\lambda}_2 + \vec{b}_3, & \log\left(\vec{\sigma}(\vec{w})\right) = W_4\vec{\lambda}_2 + \vec{b}_4. 
\end{eqnarray*}

\textbf{Sampled Gradients:} One can directly compute the gradients with respect to the generative parameters $\Theta$, including $t$, $S$ and $W$, and the sample gradients with respect to the variational parameters $\Phi$, including $\vec{\mu}(\vec{w})$ and $\vec{\sigma}(\vec{w})$. Moreover, applying the reparameterization tricks yields
\[
\partial L/\partial \vec{\mu}(\vec{w}) \approx \partial L/\partial \hat{\vec{\theta}}, \quad \partial L/\partial \vec{\sigma}(\vec{w}) \approx \hat{\epsilon}\cdot\partial L/\partial \hat{\vec{\theta}},  
\]
which can be used to jointly update $\Theta$ and $\Phi$ by stochastic gradient backpropagation.

\section{Experiment}\label{sec:experiment}
In this section, we investigate the effectiveness of NSMDM and NSMTM on large-scale collections of short text, especially tweets. The objective of the experiments includes:  (1) a quantitative evaluation of predictive performance and topic coherence; (2) a quantitative measurement of latent topic sparsity; (3) a quantitative evaluation of the regularization parameter and learning rate; and (4) an interpretation of inferred topics.

\subsection{Data sets}
We conduct the experiments on three different genres of large-scale real-world text corpora. To make a direct comparison with the existing work, we adopt the same pre-processing setup as~\cite{Miao-2016-Neural, Srivastava-2017-Autoencoding, Card-2017-Neural}. 

\begin{itemize}
\item \textbf{Twitter. } Tweets are good examples of short user-generated content. We collect three collections of tweets from the Twitter data set released by the Stanford Network Analysis Project\footnote{http://snap.stanford.edu/data/twitter7.html}. The original data set contains 467 million Twitter posts from 20 million users covering the period from June 1 2009 to December 31 2009. Three sampled Twitter data sets, namely Twitter (S), Twitter (M) and Twitter (L), are the collections of tweets with short, medium and long average document length by words, respectively. 

\item \textbf{NYT. } The collection of New York Time articles\footnote{https://ldc.upenn.edu/} is a good representative of user-generated content. The original dataset contains 299,752 news articles published in New York Times between 1987 and 2007, where the vocabulary size is 102,660 and the average length of each document is 166.1. To investigate the performance of all the methods on short content, we vary the document length by randomly sampling words from the original document and obtain three short text corpus, denoted as NYT (S), NYT (M) and NYT (L).

\item \textbf{20NG. } This data set, denoted as 20NG\footnote{http://qwone.com/$\sim$jason/20Newsgroups/}, contains 18,774 newsgroup documents labeled in 20 categories, with a vocabulary of 60,698 unique words. We use the sampled data set with 11,000 training instances and 2000 word vocabulary \cite{Srivastava-2017-Autoencoding} and vary the document length by randomly sampling words from the original document. As a result, we obtain two short text corpora, denoted as 20NG (S) and 20NG (M). 
\end{itemize}
The statistics of all eight data sets are summarized in Table~\ref{tab:datasets}.
\begin{table}[t]
\centering 
\caption{Statistics of All Data Sets} \label{tab:datasets}
\begin{tabular}{|c|c|c|c|} \hline
\multirow{2}{*}{Data set} & \multirow{2}{*}{\# Documents} & Vocabulary & Avg doc len \\
& & Size & by words \\ \hline
Twitter (S) & 54,000,648 & 74,027 & 6.7 \\ \hline
Twitter (M) & 4,470,965 & 71,497 & 11.3 \\ \hline
Twitter (L) & 243,472 & 48,590 & 16.0 \\ \hline
NYT (S) & 279,815 & 66,317 & 7.1 \\ \hline
NYT (M) & 298,714 & 81,212 & 14.2 \\ \hline
NYT (L) & 297,456 & 87,969 & 21.2 \\ \hline
20NG (S) & 14,925 & 1,965 & 8.3 \\ \hline
20NG (M) & 9,763 & 1,982 & 16.5 \\ \hline
\end{tabular}
\end{table}
\subsection{Metrics}
We compare NSMDM and NSMTM with other methods by \textit{perplexity} and \textit{pointwise mutual information (PMI)}, which are the standard criteria for measuring the quality of document and topic models. The results obtained by using some other metrics~\cite{Banerjee-2002-Adapted, Bouma-2009-Normalized, Chang-2009-Reading} are similar and hence omitted due to the page limit. 

\begin{definition}[Perplexity \cite{Blei-2003-Latent}]
The perplexity is used to measure the predictive performance of document/topic model. Mathematically, given $D_{train}$ and $D_{test}$ with each document $\vec{w}_j$ in $D_{test}$ divided into two parts, $\vec{w}_j=(\vec{w}_{j1}, \vec{w}_{j2})$, the perplexity is calculated as: 
\begin{equation}
\text{Perplexity} = \exp\left\{-\frac{\sum_{j\in D_{test}}\text{log }p(\vec{w}_{j2}|\vec{w}_{j1}, D_{train})}{\sum_{j\in D_{test}}|\vec{w}_{j2}|}\right\},
\end{equation}
where $|\vec{w}_{j2}|$ is the number of tokens in $\vec{w}_{j2}$.
\end{definition}
We follow~\cite{Miao-2016-Neural} by using the original variational lower bound to estimate the test document perplexities of all the models and held out 10$\%$ documents as test set $D_{test}$ for all data sets. 
\begin{definition}[PMI \cite{Newman-2010-Automatic}]
The PMI score is used to measure the semantic coherence of inferred topics. Mathematically, the PMI score of a topic $\vec{\phi}_k$ refers to the average relevance of each pair of the top-$N$ words:
\begin{equation}\label{Def:PMI}
\text{PMI}(\vec{\phi}_k) = \frac{2}{N(N-1)}\sum\limits_{1\leq i<j\leq N}\log\left(\frac{p(w_i,w_j)}{p(w_i)p(w_j)}\right),
\end{equation}
where $p(w_i,w_j)$ is the probability that $w_i$ and $w_j$ occurs in the same document and $p(w_i)$ is the probability that $w_i$ appears in a document. 
\end{definition}
These probabilities are computed from a much larger corpus. In this paper, we set $N=15$ throughout.
\begin{definition}[Topic Sparsity \cite{Wang-2009-Decoupling}]
The topic sparsity (TS) score is used to measure the topic sparsity in document-topic and topic-word distributions quantitatively. Mathematically, the TS scores of $\vec{\theta}_j$ and $\vec{\phi}_k$ are
\begin{equation}\label{Def:TS}
\text{TS}(\vec{\theta}_j) = \frac{\sum_{k=1}^{K} 1_{({\theta}_{jk}=0)}}{K}, \quad \text{TS}(\vec{\phi}_k) = \frac{\sum_{v=1}^{|V|}1_{({\phi}_{kv}=0)}}{|V|},
\end{equation}
where $K$ is the number of topics and $V$ is the vocabulary.
\end{definition}
\begin{remark}
Note that the definition here is different from \cite{Lin-2014-Dual, Lin-2016-Understanding} for topic sparsity: ours is directly defined by topic proportion while \cite{Lin-2014-Dual, Lin-2016-Understanding} use an unnatural scheme with a set of auxiliary topic selectors. 
\end{remark}

\begin{table*}[t]
\caption{Performance of All Algorithms on All Data Sets.}\label{tab:result_perplexity_pmi}
\begin{tabular}{|c||c|c||c|c||c|c||c|c|} \hline
& Perplexity & PMI & Perplexity & PMI & Perplexity & PMI & Perplexity & PMI \\ \hline \hline
& \multicolumn{2}{c||}{Twitter (S)} & \multicolumn{2}{c||}{Twitter (M)} & \multicolumn{2}{c||}{Twitter (L)} & \multicolumn{2}{c|}{NYTimes (S)} \\ \hline
NSMTM & 7103.86 & \textbf{0.60} & 4951.12 & \textbf{0.948} & 2933.75 & 1.14 & 9131.90 & \textbf{0.51}  \\ \hline
NSMDM & \textbf{5572.39} & 0.48 & \textbf{2976.13} & 0.59 &\textbf{1384.98} & 0.75 & \textbf{7783.46} & 0.39 \\ \hline
NVDM & 6019.66 & 0.34 & 3400.97 & 0.47 & 1634.62 & 0.62 & 8007.03 & 0.27 \\ \hline
NVTM & 6170.95 & 0.31 & 3382.25 & 0.49 & 1901.80 & 0.58 & 8026.69 & 0.24 \\ \hline
ProdLDA & 7181.58 & 0.56 & 5227.23 & 0.945 & 3016.33 & \textbf{1.18} & 10322.43 & 0.47 \\ \hline
AEVLDA & 7031.69 & 0.45 & 5011.76 & 0.58 & 2971.37 & 0.71 & 11415.39 & 0.36 \\ \hline
OLDA & 15536.84 & 0.42 & 10512.09 & 0.55 & 4193.53 & 0.72 & 12566.67 & 0.33 \\ \hline
OBTM & - & 0.36 & - & 0.53 & - & 0.47 & - & 0.35 \\ \hline \hline
& \multicolumn{2}{c||}{NYT (M)} & \multicolumn{2}{c||}{NYT (L)} & \multicolumn{2}{c||}{20NG (S)} & \multicolumn{2}{c|}{20NG (M)} \\ \hline
NSMTM & 9497.21 & 0.58 & 9320.59 & \textbf{0.69} & 1262.66 & \textbf{0.41} & 1195.18 & \textbf{0.474} \\ \hline
NSMDM & \textbf{8276.35} & 0.42 & 8195.83 & 0.47 & \textbf{923.49} & 0.34 & \textbf{883.47} & 0.37 \\ \hline
NVDM & 8403.79 & 0.39 & \textbf{8023.84} & 0.40 & 940.00 & 0.29 & 892.46 & 0.32 \\ \hline
NVTM & 9349.62 & 0.39 & 8932.37 & 0.41 & 1155.09 & 0.27 & 929.59 & 0.29 \\ \hline
ProdLDA & 11954.86 & \textbf{0.59} & 11036.25 & 0.65 & 1557.20 & 0.38 & 1423.08 & 0.470 \\ \hline
AEVLDA & 10924.48 & 0.45 & 10776.58 & 0.52 & 1364.73 & 0.31 & 1385.38 & 0.36 \\ \hline
OLDA & 17092.26 & 0.43 & 17825.47 & 0.51 & 1768.18 & 0.34 & 1586.40 & 0.38 \\ \hline
OBTM & - & 0.37 & - & 0.41 & - & 0.29 & - & 0.33 \\ \hline
\end{tabular}
\end{table*}
\begin{figure*}[t]
\includegraphics[width=0.32\textwidth]{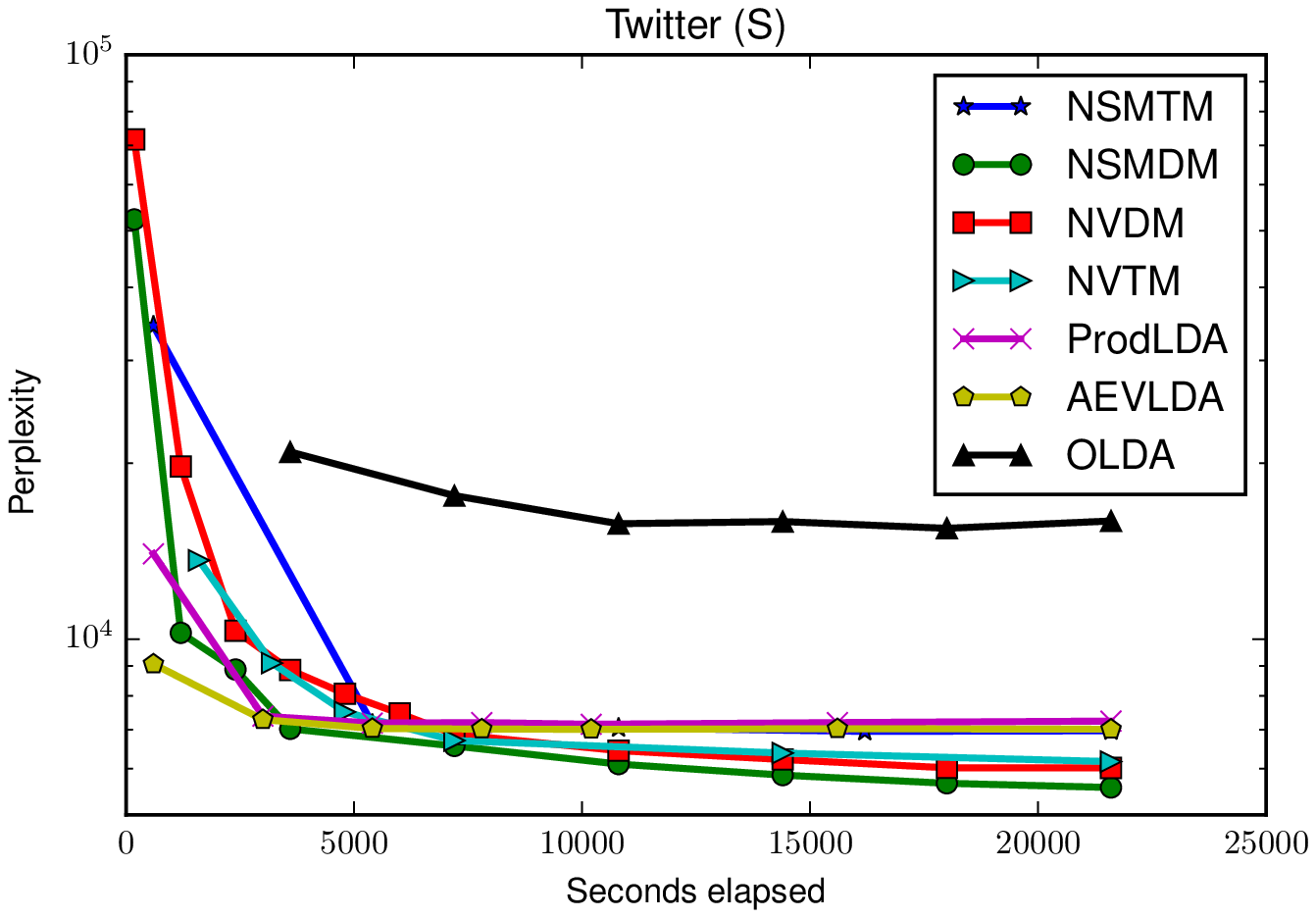}
\includegraphics[width=0.32\textwidth]{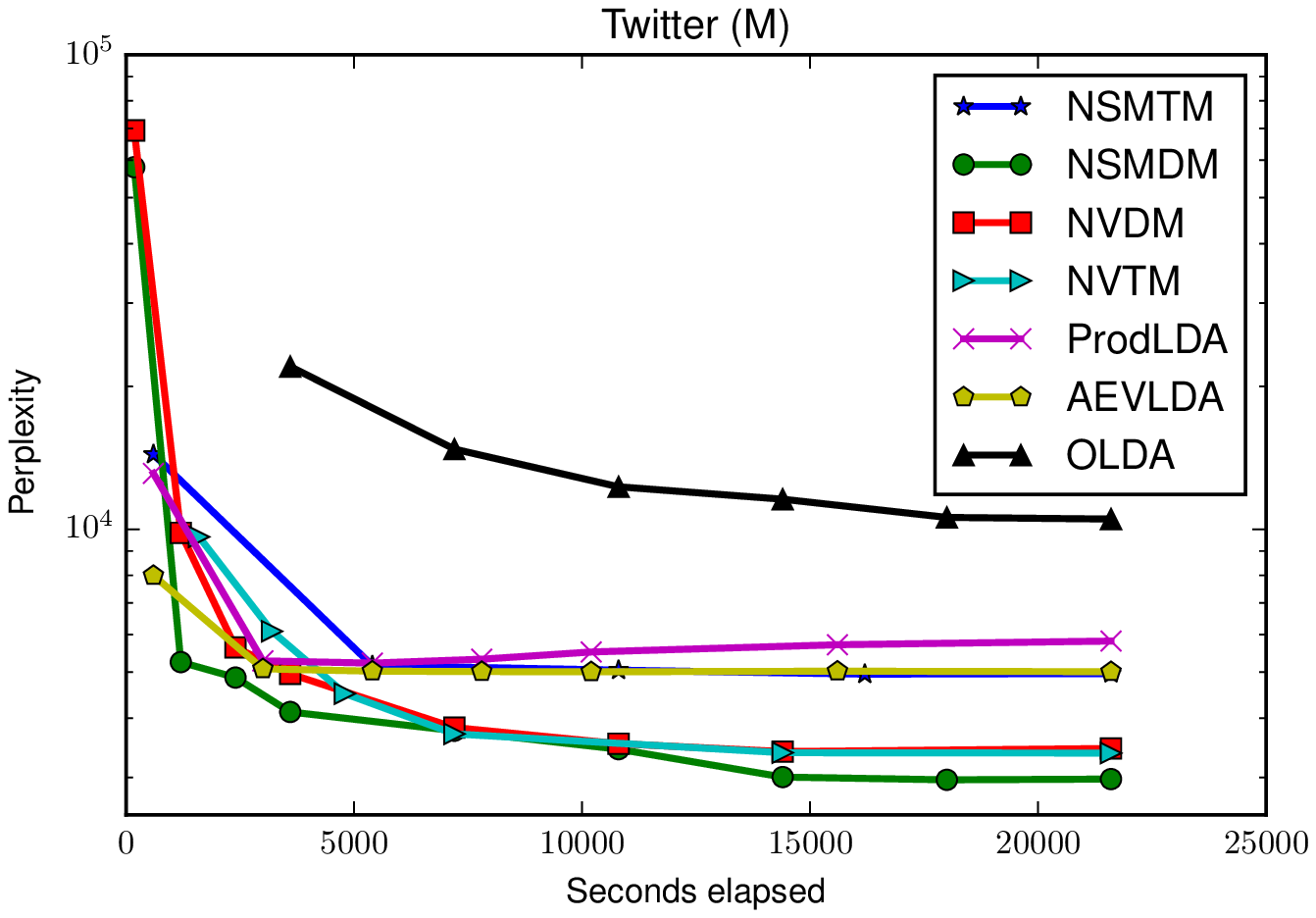}
\includegraphics[width=0.32\textwidth]{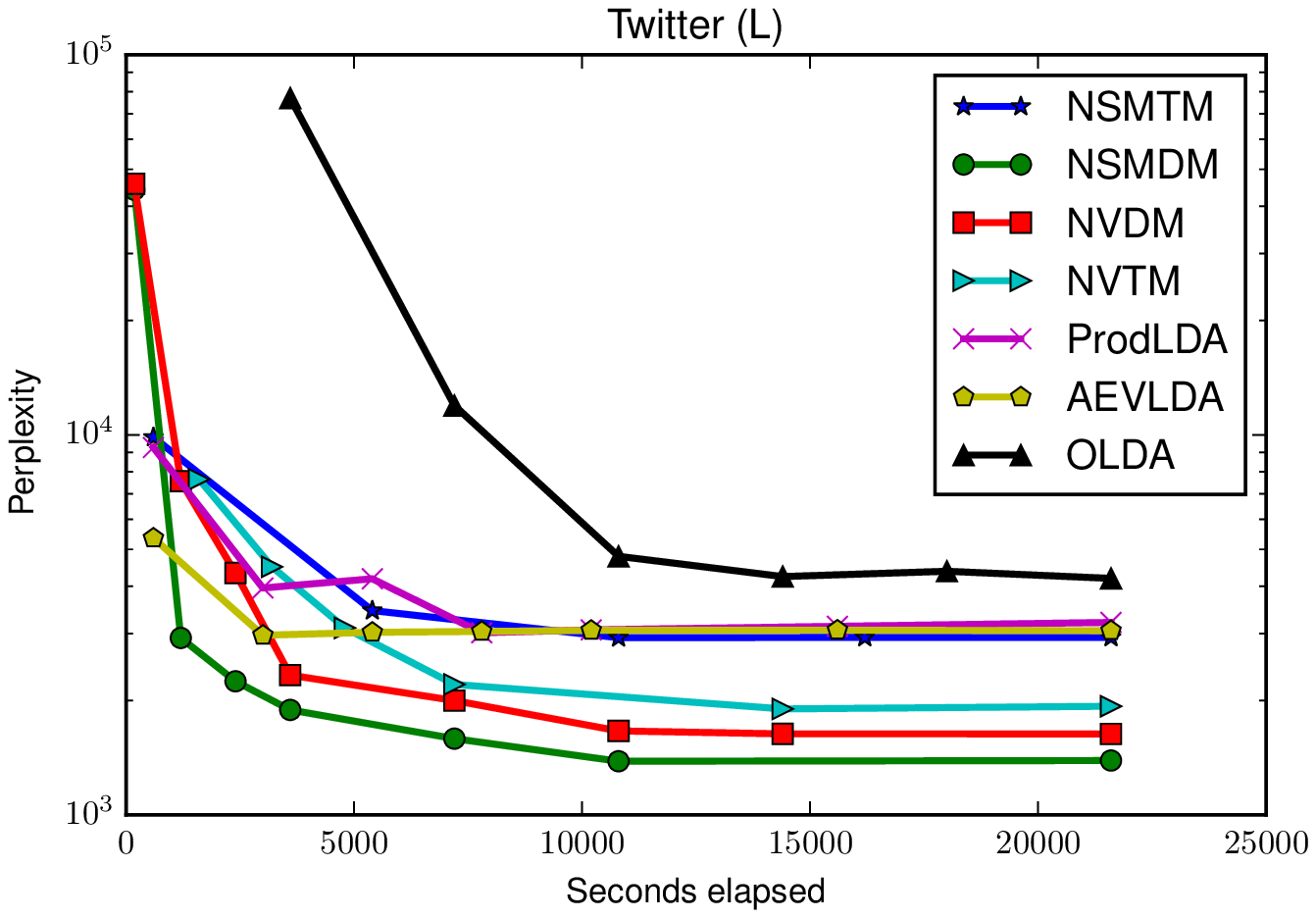}
\includegraphics[width=0.32\textwidth]{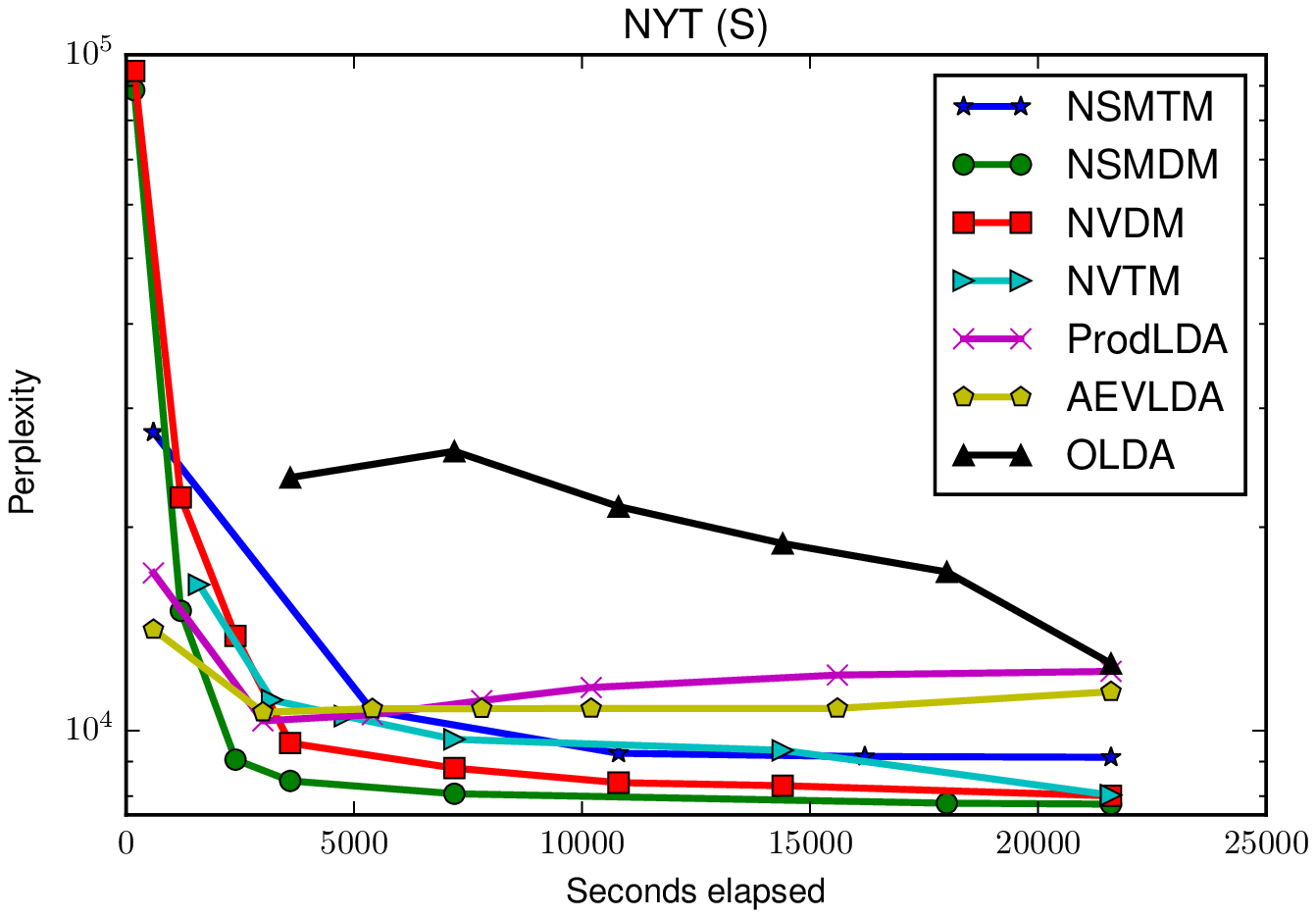}
\includegraphics[width=0.32\textwidth]{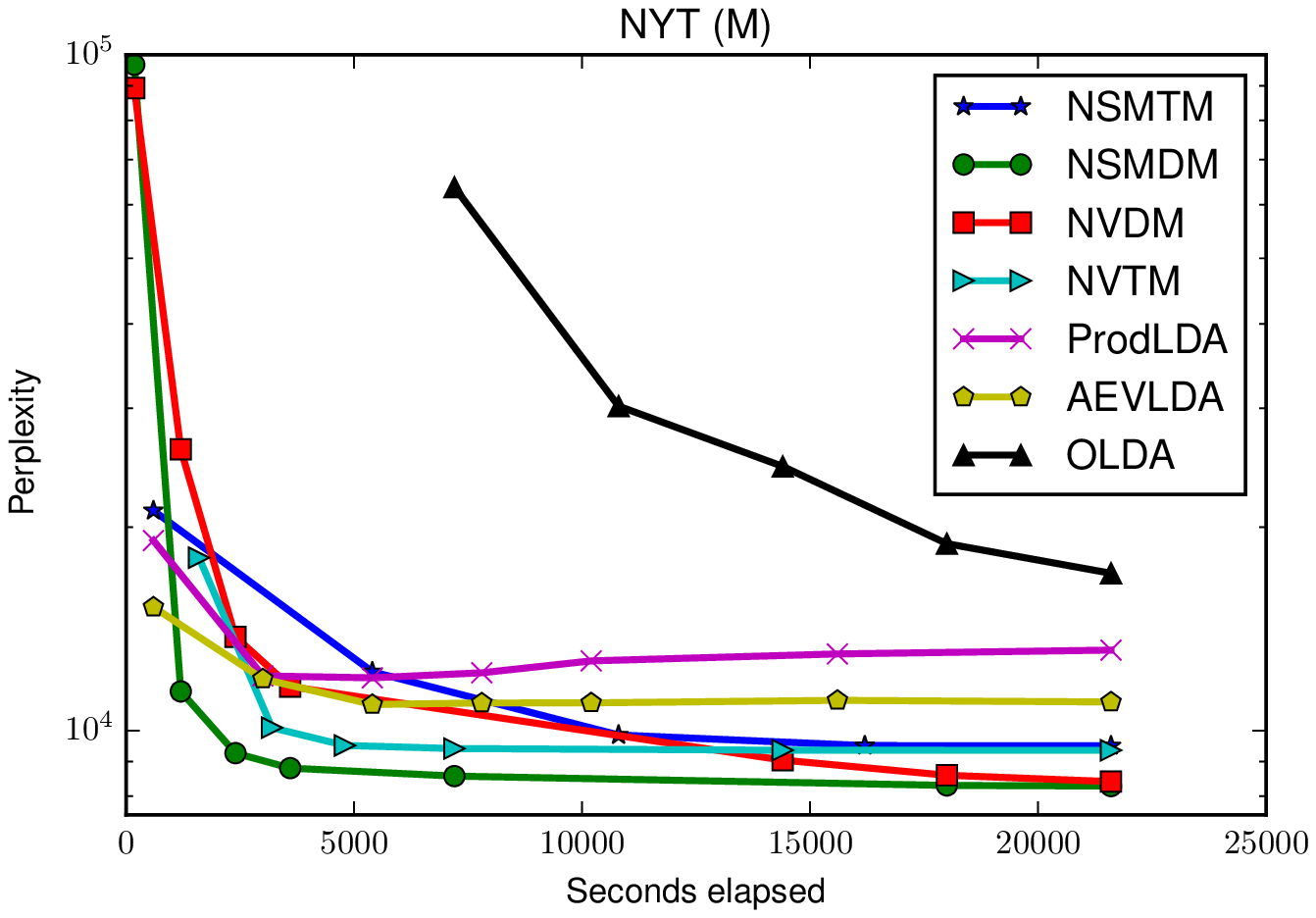}
\includegraphics[width=0.32\textwidth]{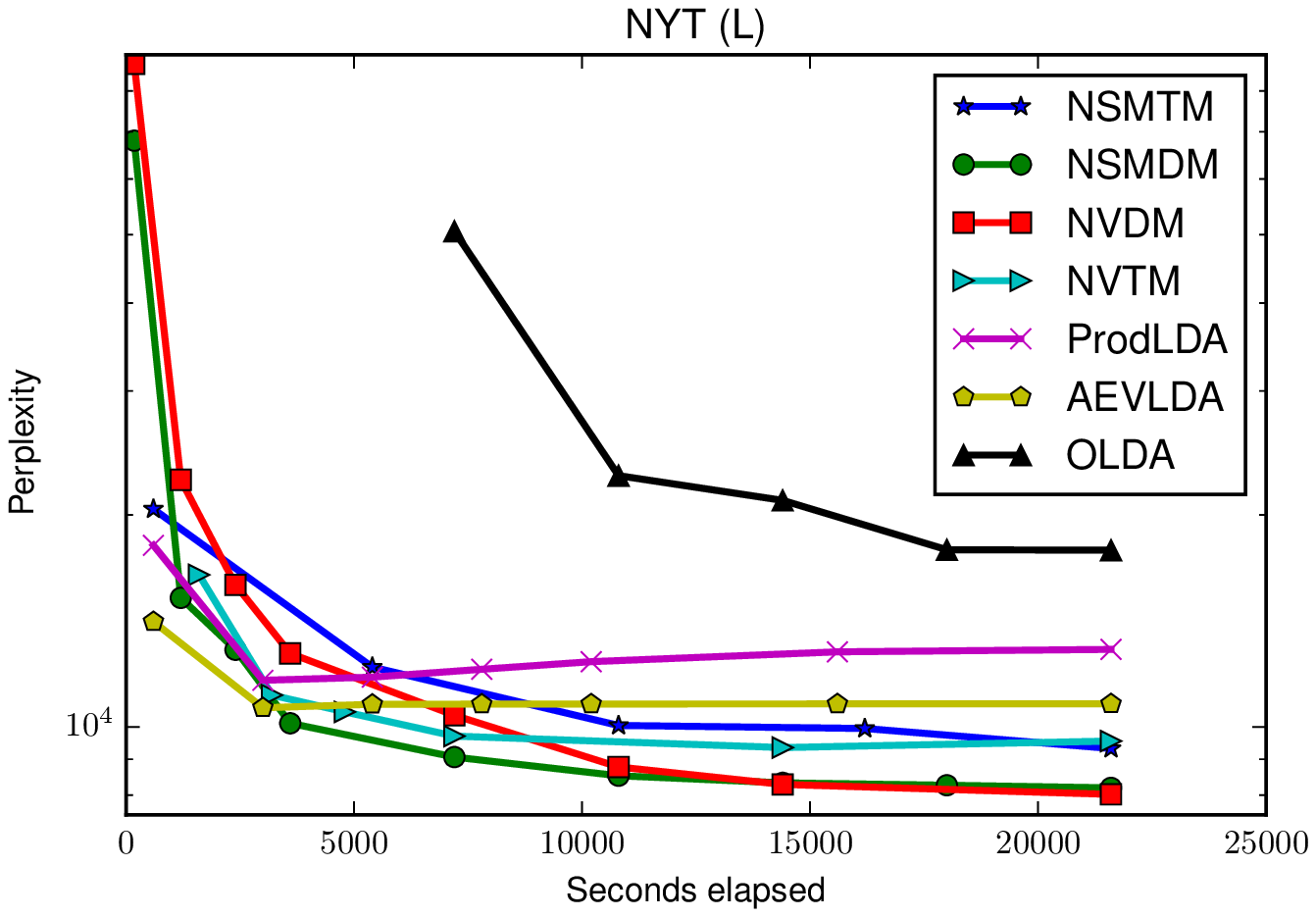}
\caption{Perplexity vs Time on Twitter and NYT Data Sets} \label{fig:perplexity-time}
\end{figure*}
\subsection{Candidate Algorithms for Comparison}
We compare NSMDM and NSMTM with the following two probabilistic topic models and four deep neural document/topic models. \begin{itemize}
\item \textbf{OLDA.} Online LDA \cite{Hoffman-2010-Online} induces topic sparsity as the hyperparameter approaches zero. We use the implementation\footnote{https://github.com/blei-lab} provided by the authors.
\item \textbf{OBTM.} Online Biterm Topic Model  \cite{Cheng-2014-Biterm} is a sparsity-enhanced probabilistic topic model that performs well on short text. We use the implementation\footnote{https://github.com/xiaohuiyan/OnlineBTM} with incremental Gibbs sampling provided by the authors.
\item \textbf{NVDM.} Neural Variational Document Model \cite{Miao-2016-Neural} is an unsupervised generative document model that has been proven better than many existing models, including RSM \cite{Hinton-2009-Replicated}, docNADE \cite{Larochelle-2012-Neural} and SBN/DARN \cite{Mnih-2014-Neural}. We use the implementation\footnote{https://github.com/ysmiao/nvdm} provided by the authors.
\item \textbf{AEVLDA/ProdLDA.} Autoencoding variational LDA \cite{Srivastava-2017-Autoencoding} provides an Autoencoding variational inference framework for topic model. ProdLDA is the improvement of AEVLDA by replacing the mixture structure with a product of experts using the neural network. We use the implementation\footnote{https://github.com/akashgit/autoencoding{\_}vi{\_}for{\_}topic{\_}models} provided by the authors.
\item \textbf{NVTM.} Neural Variational Topic Model \cite{Card-2017-Neural} is a neural sparse additive generative model that induces topic sparsity, outperforming its probabilistic counterpart \cite{Eisenstein-2011-Sparse}. We use the implementation\footnote{https://github.com/dallascard/neural\_topic\_models} provided by the authors. 
\end{itemize}
Many sparsity-enhanced topic models, such as sparse topic models \cite{Shashanka-2008-Sparse, Wang-2009-Decoupling}, dual-sparse topic model \cite{Lin-2014-Dual}, sparse coding \cite{Hoyer-2004-Matrix, Zhu-2011-Sparse}, and focused topic models \cite{Chen-2012-contextual}, are based on specific batch sampling and variational inference methods and hence can not scale to large text corpora\footnote{The alternative way of sampling a small batch of the entire data set is, unfortunately proven to result in high perplexity and misleading inferred topics \cite{Hoffman-2013-Stochastic}.}. Furthermore, \cite{Zhang-2013-Sparse, Lin-2016-Understanding} only identify sparsity in either topic mixtures or topic-word distributions. Thus, we exclude these methods in our experiment. We also exclude the neural methods proposed in \cite{Hinton-2009-Replicated, Larochelle-2012-Neural, Mnih-2014-Neural, Miao-2017-Discovering} since they have been proven worse than NVDM, ProdLDA and NVTM \cite{Miao-2016-Neural, Srivastava-2017-Autoencoding, Card-2017-Neural}.  

In the experiment, we use the default parameters for probabilistic models, and set the same multilayer perceptron (MLP) and dropout on the output of the MLP for all neural models on each data set for a fair comparison. Moreover, we set the number of topics $T = \{50, 150, 200\}$ for 20NG, NYT and Twitter data sets, respectively.

For the computing environment, we run two probabilistic models with Intel Xeon CPU E5-2643 v2@3.50GHz CPU on all data sets, and the other neural models including NSMDM and NSMTM with NVidia Titan Xp GPU (12GB memory). It is worthy noting that the GPU implementation for the probabilistic models is possible but still under exploration~\cite{Yan-2009-Parallel}. In addition, we set the parallel setting with dual GPU for the largest Twitter (S) data set while the standard setting with a single GPU for the other data sets. Each model is trained within 1 hour for 20NG data sets and 6 hours for NYT and Twitter data sets. 

\subsection{Experimental Results}
We first present and analyze the performance of all methods, and then demonstrate the existence of topic sparsity in the latent structure of held-out documents. Next we carry out the parameter sensitivity analysis by tuning the regularization parameter $\gamma$ and the learning rate $\eta$. Finally, we interpret some selected topics.  
\begin{table*}[t]
\caption{Average Topic Sparsity on All Data Sets. } \label{tab:result_topic_sparsity}
\begin{tabular}{|c||c|c||c|c||c|c||c|c|} \hline
& TS($\vec{\theta}$)  & TS($\vec{\phi}$) & TS($\vec{\theta}$)  & TS($\vec{\phi}$) & TS($\vec{\theta}$)  & TS($\vec{\phi}$) & TS($\vec{\theta}$)  & TS($\vec{\phi}$) \\ \hline \hline
& \multicolumn{2}{c||}{Twitter (S)} & \multicolumn{2}{c||}{Twitter (M)} & \multicolumn{2}{c||}{Twitter (L)} & \multicolumn{2}{c|}{NYTimes (S)} \\ \hline
NSMTM & 0.9928 & 0.9714 & 0.9925 & 0.9707 & 0.9917 & 0.9628 & 0.9829 & 0.9585 \\ \hline
NSMDM & 0.9927 & 0.8301 & 0.9921 & 0.8282 & 0.9913 & 0.8267 & 0.9840 & 0.7863 \\ \hline \hline
& \multicolumn{2}{c||}{NYT (M)} & \multicolumn{2}{c||}{NYT (L)} & \multicolumn{2}{c||}{20NG (S)} & \multicolumn{2}{c|}{20NG (M)} \\ \hline
NSMTM & 0.9797 & 0.9642 & 0.9711 & 0.9650 & 0.9441 & 0.8100 & 0.9380 & 0.8104 \\ \hline
NSMDM & 0.9773 & 0.8014 & 0.9659 & 0.8012 & 0.9383 & 0.5703 & 0.9224 & 0.5655 \\ \hline
\end{tabular}
\end{table*}
\begin{figure*}[!ht]
\includegraphics[width=0.24\textwidth]{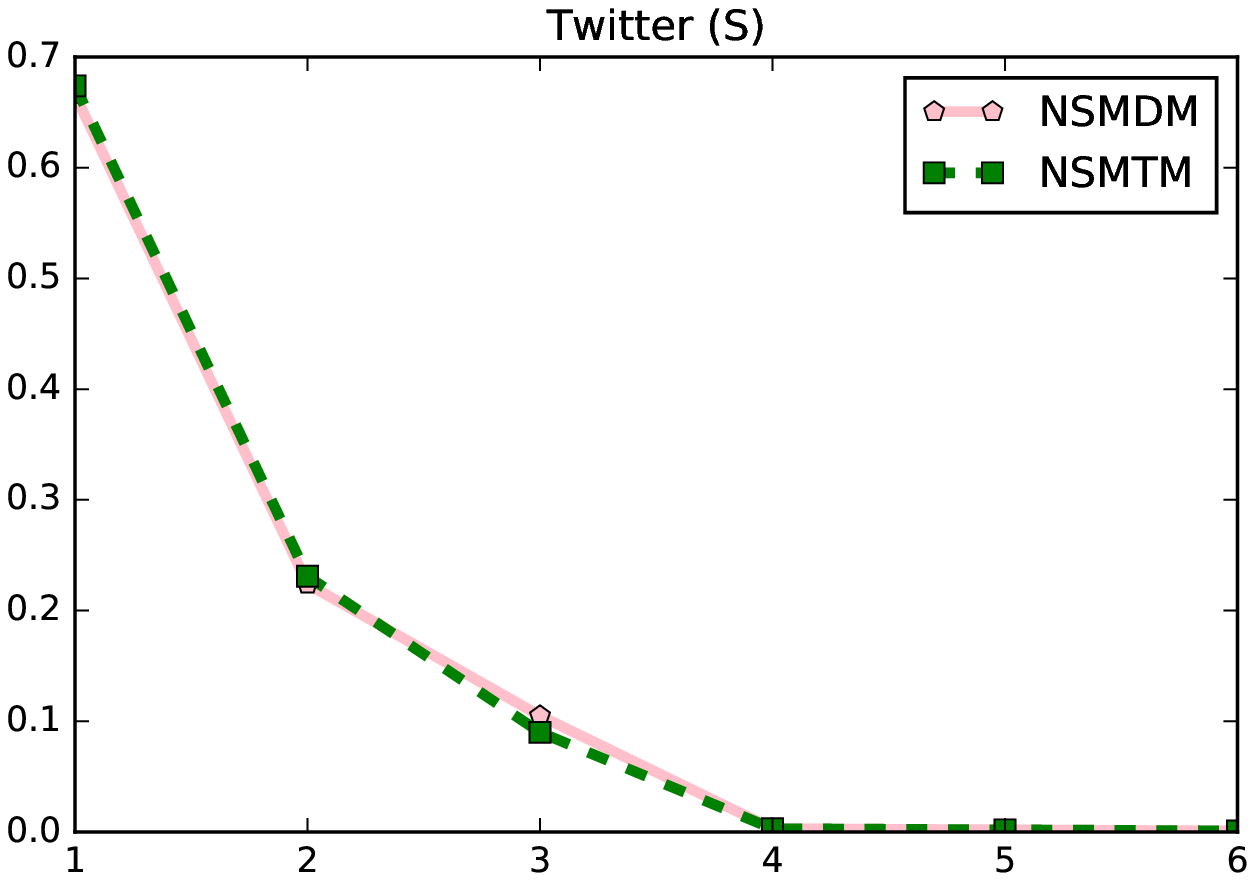}
\includegraphics[width=0.24\textwidth]{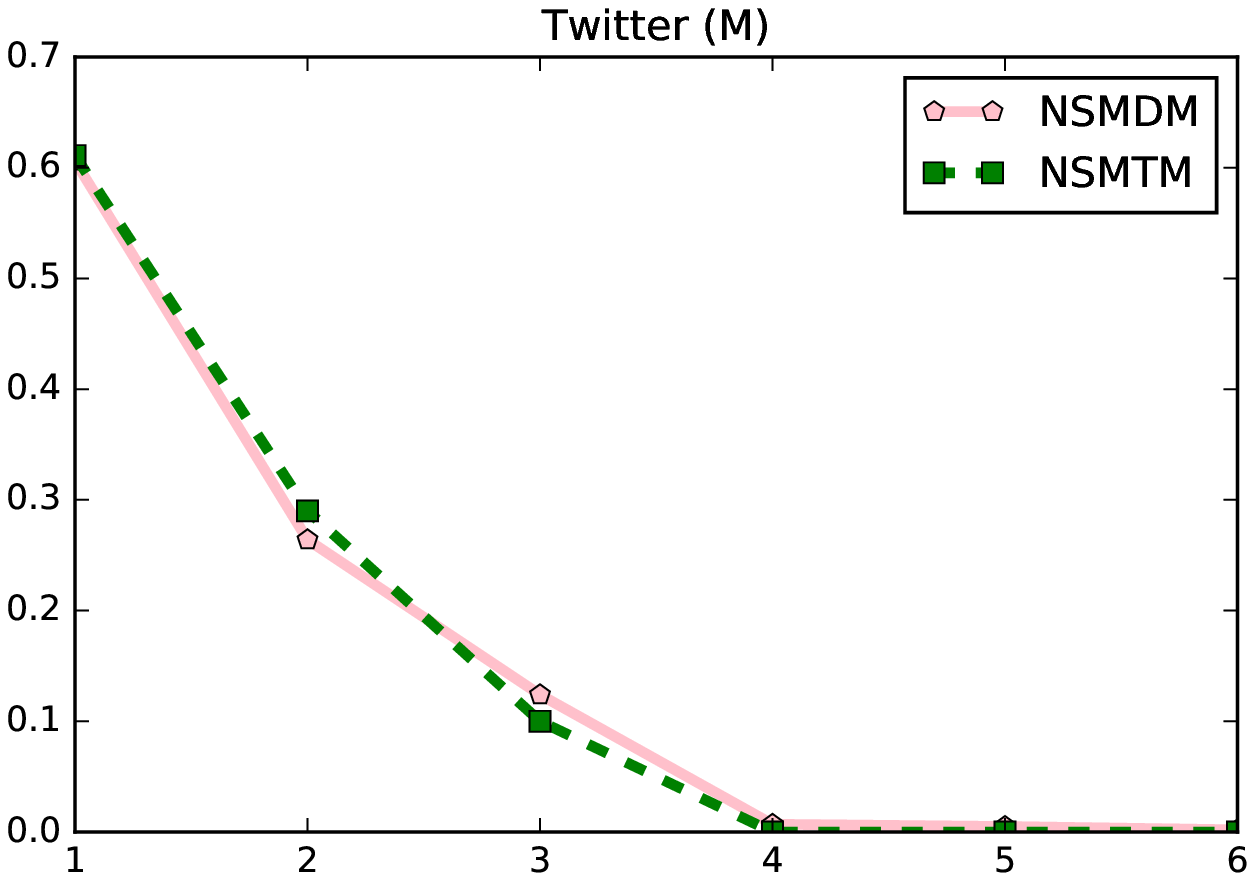}
\includegraphics[width=0.24\textwidth]{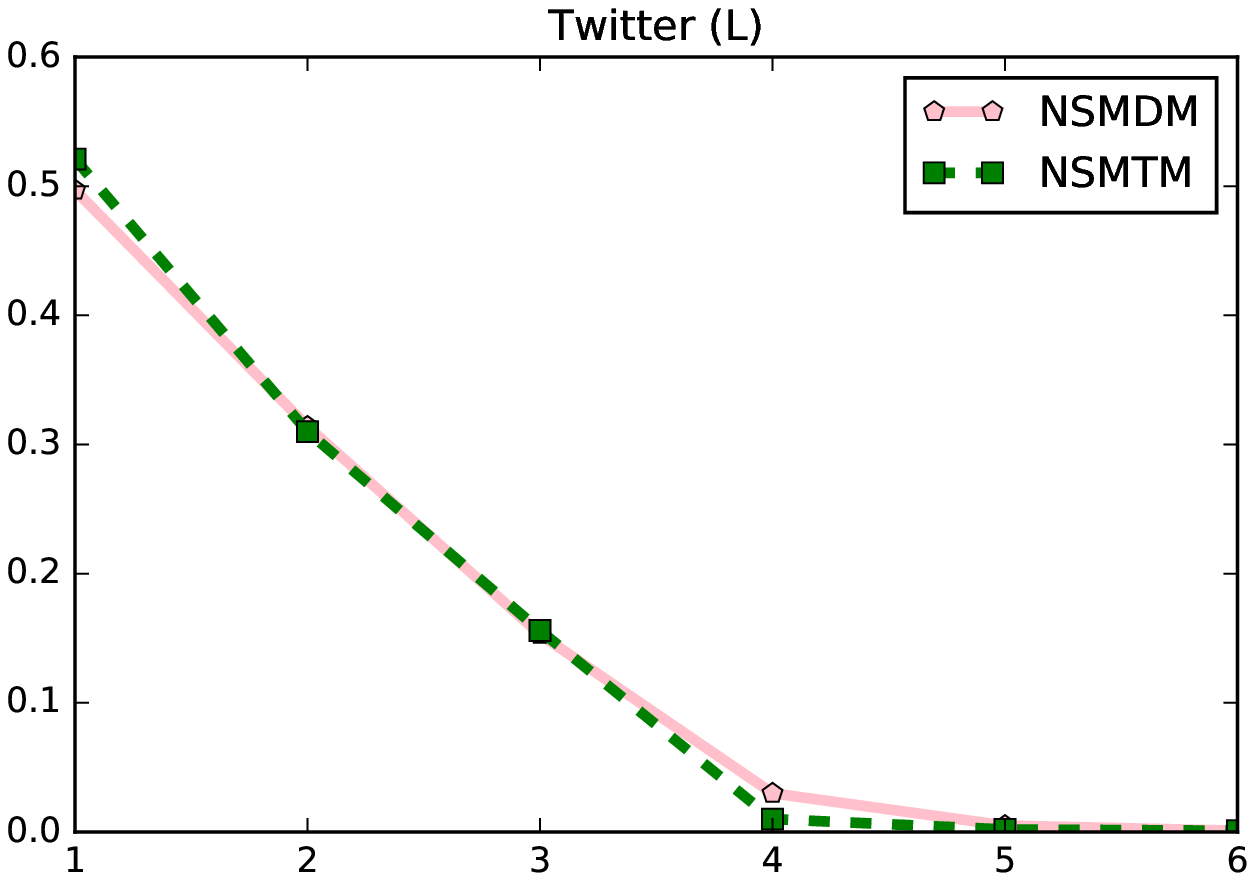}
\includegraphics[width=0.24\textwidth]{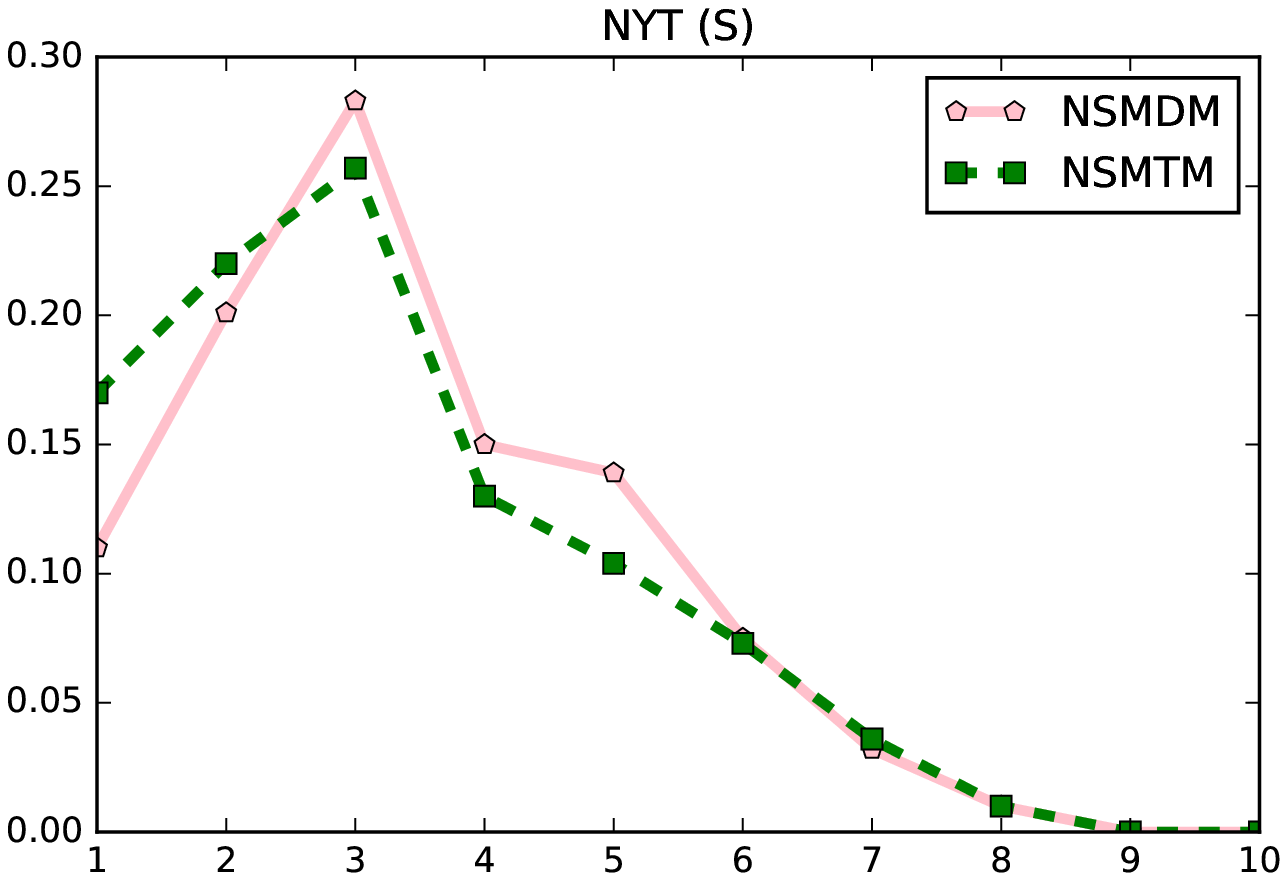}
\includegraphics[width=0.24\textwidth]{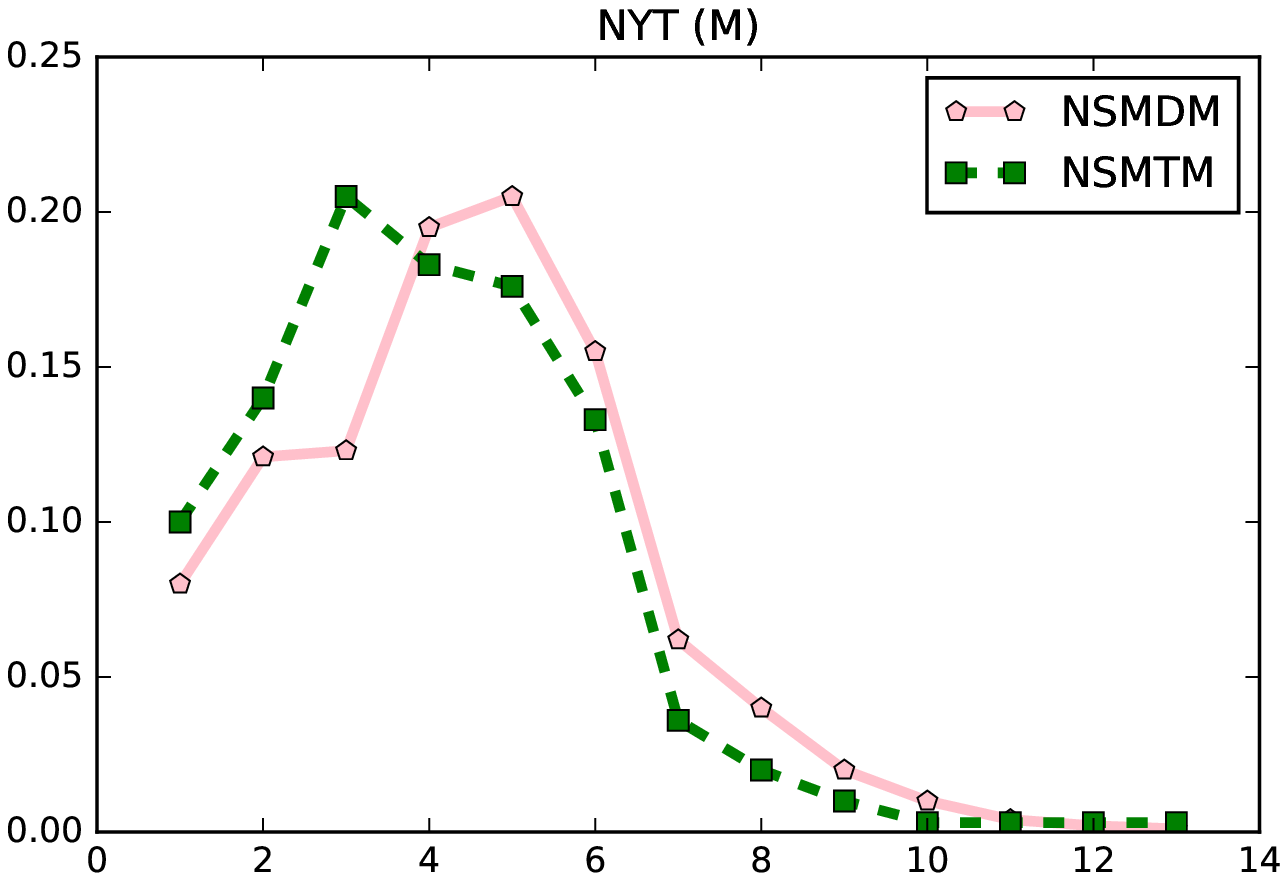}
\includegraphics[width=0.24\textwidth]{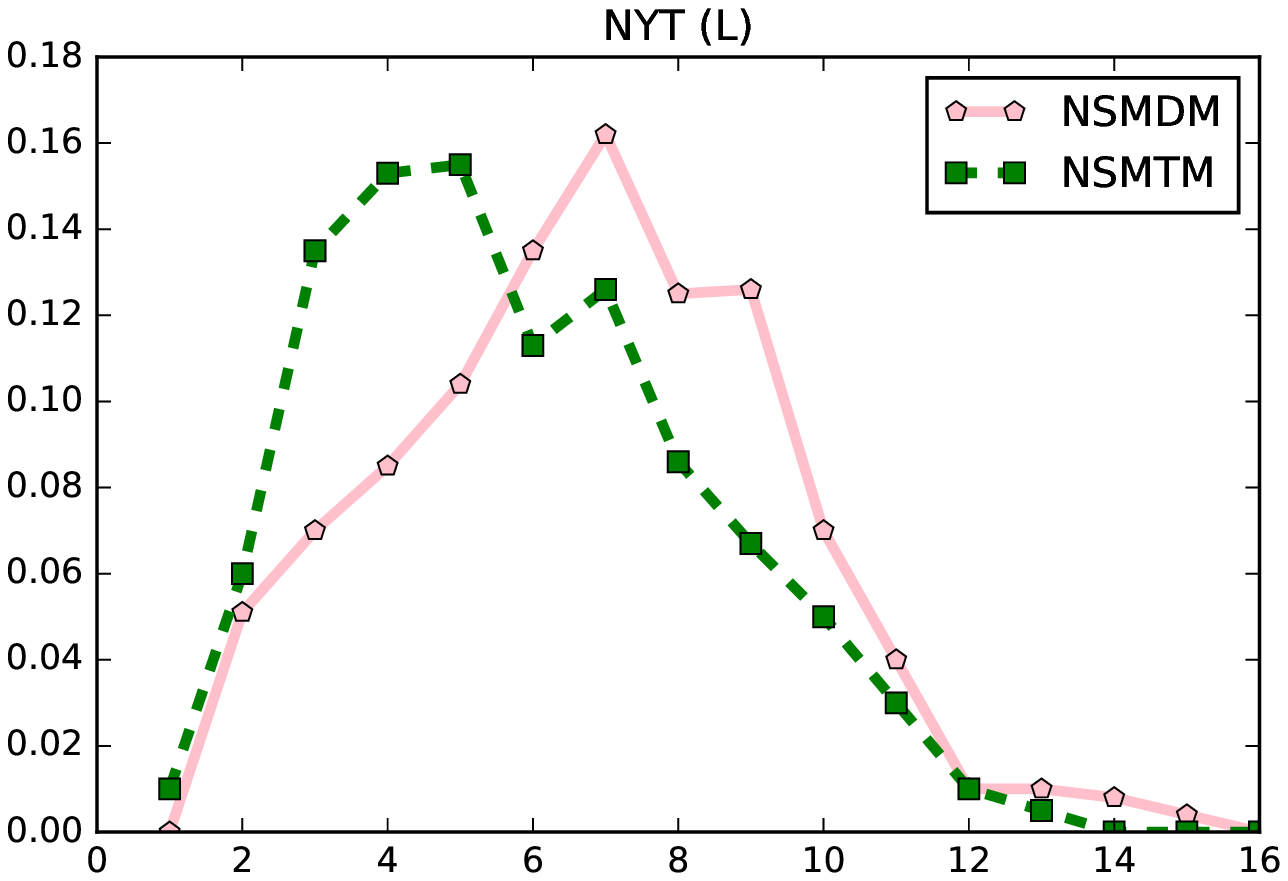}
\includegraphics[width=0.24\textwidth]{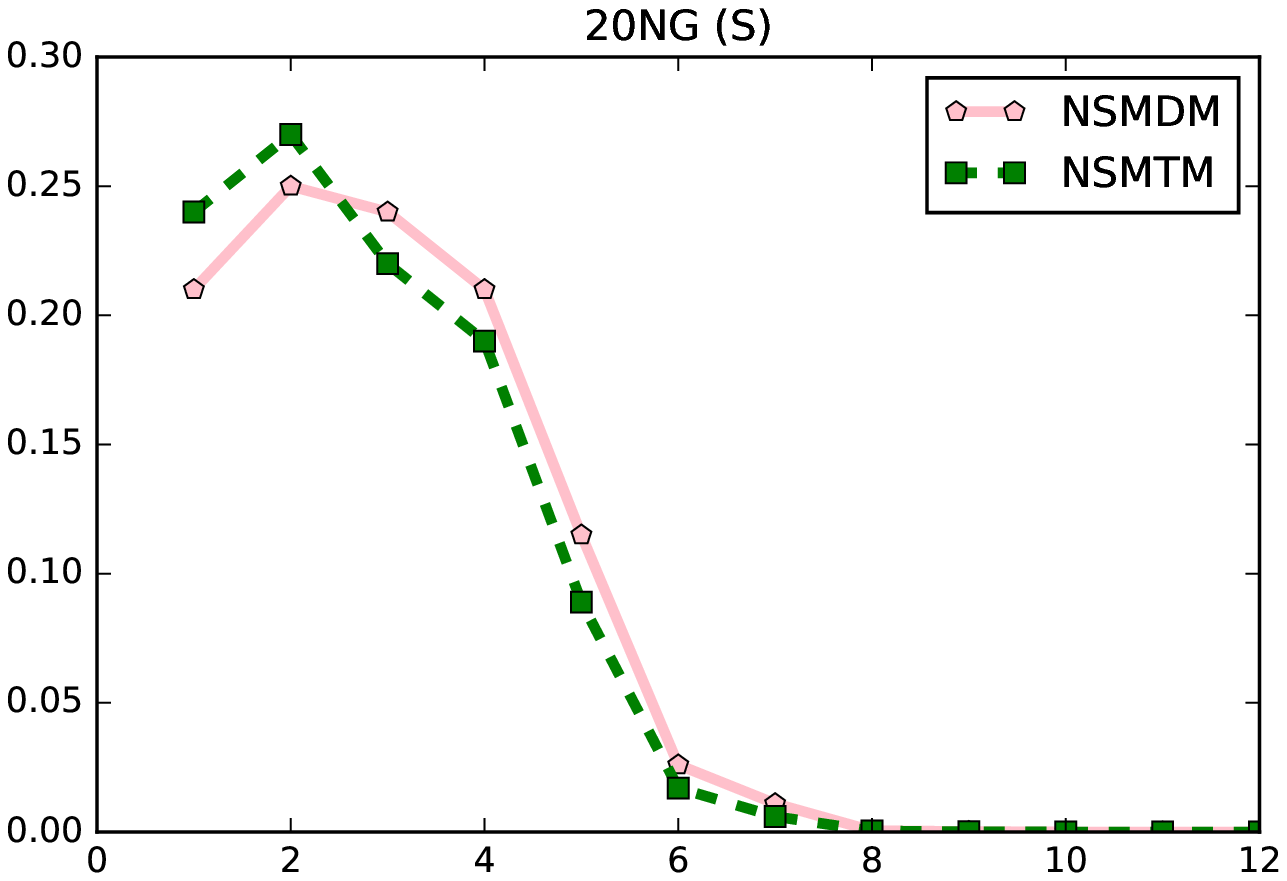}
\includegraphics[width=0.24\textwidth]{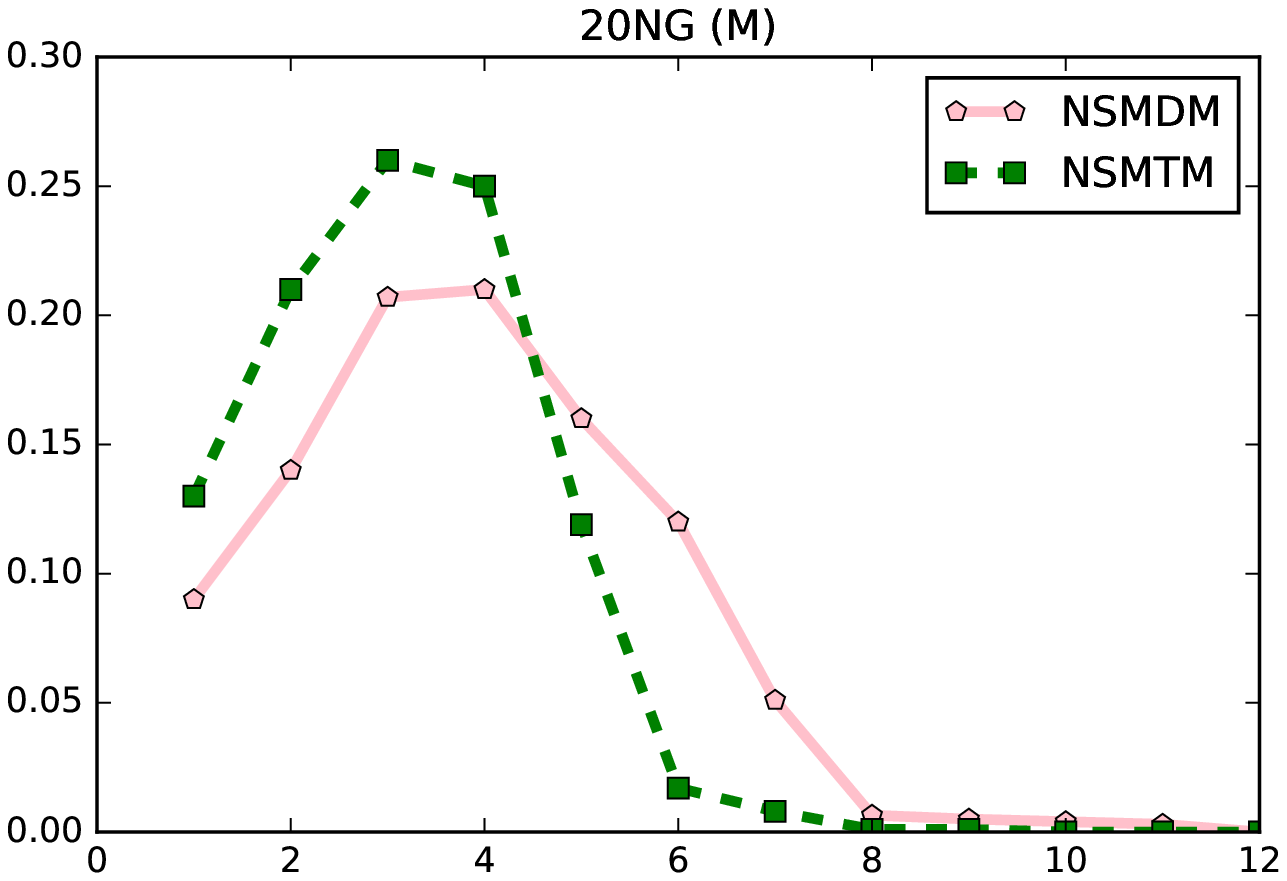}
\caption{Topic Sparsity in Held-out Documents of All Data Sets}\label{fig:topic-sparisty-All}
\end{figure*}

\subsubsection{Predictive Performance and Topic Coherence}
The predictive performance and topic coherence of all methods are summarized in Table~\ref{tab:result_perplexity_pmi}, where the perplexity of OBTM is not available since it is not based on the generative modeling\cite{Cheng-2014-Biterm}. 

\textbf{Twitter.} We observe that NSMTM yields the highest PMI score followed by ProdLDA while NSMDM yields the lowest perplexity followed by NVDM and NVTM. Possible explanations include (i) NSMTM and NSMDM can identify sparse topical structure of short text, (ii) NSMTM and ProdLDA explicitly model latent topics, (ii) NSMDM, NVDM and NVTM build up a simpler model structure with less parameters than neural topic models, hence yield better generalization results. The poor performance of probabilistic models may be due to the faster training with GPU over CPU, supporting the importance of the neural variational inference for analyzing large text corpora. To further investigate efficiency of NSMTM and NSMDM, we present the perplexity as a function of time in Figure~\ref{fig:perplexity-time}. We observe that NSMDM outperforms other methods consistently, supporting the necessity of modeling topic sparsity for short text corpora. In addition, the poor performance of OBTM illustrates that part of tweets may cover multiple topics, and addressing general topic sparsity is helpful for analyzing online streaming tweets.

\textbf{NYT.} NSMDM and NSMTM achieve the lowest perplexity and highest PMI score, respectively, and outperform other methods on nearly all data sets except for NYT (L). This can be explained by the weak topic sparsity in NYT (L) since the length of documents is long in NYT (L), which is close to a traditional social media data set. To further investigate efficiency of NSMTM and NSMDM, we also present the perplexity as a function of time in Figure~\ref{fig:perplexity-time}. The performance of all methods become worse on NYT, which suggests that mining topics is more difficult on NYT than Twitter. Nonetheless, the best performance of NSMDM and NSMTM provides a strong evidence that they can work well with user-generated contents. 

\textbf{20NG.} We observe that NSMDM and NSMTM are again best on 20NG in terms of perplexity and PMI score, respectively. 20NG is a relatively normal text collection with smaller vocabulary size where each document provides sufficient statistics of word co-occurrence. As a result, the performance of all the methods become better on 20NG than NYT and Twitter. The best performance of NSMDM and NSMTM also suggests that our models can address topic sparsity in the collection of relatively normal text.

\subsubsection{Topic Sparsity}
We report the TS score of each held-out short text in Table~\ref{tab:result_topic_sparsity}, and specify the distribution of documents with respect to number of topics in Figure~\ref{fig:topic-sparisty-All}.

\textbf{Twitter.} Table~\ref{tab:result_topic_sparsity} shows that topic sparsity in Twitter is stronger than that in other data sets. Explanation: each Twitter text reflects the viewpoint of a single author while each topic concentrates on a specific social event. Figure~\ref{fig:topic-sparisty-All} provides the evidence to our explanation for Twitter: Most of tweets only contains a single topic.

\textbf{NYT.} Table~\ref{tab:result_topic_sparsity} shows that topic sparsity in NYT is weaker than Twitter. This is reasonable since NYT is a normal text collection collected from social medium. Each document covers multiple topics despite its short length. Figure~\ref{fig:topic-sparisty-All} provides the evidence to our explanation for NYT and demonstrates that the topic sparsity varies as the average length of document changes. This confirms the diversity of user-generated content in NYT and suggests that the set of topics tend to be specific as the length of document increases.

\textbf{20NG.} Table~\ref{tab:result_topic_sparsity} shows that topic sparsity in 20NG is the weakest among all the data sets. The sparsity in document-topic distribution is stronger than that in topic-word distribution, possibly because the vocabulary size is so small that a set of terms are therefore frequently used. Figure~\ref{fig:topic-sparisty-All} shows that the average number of topics in each short text is nearly three, while a majority of short text in 20NG contains 2 or 3 topics. This makes sense since each document in 20NG is a sampled news related to more than one theme.

\subsubsection{Parameter Sensitivity}
We first investigate the effect of regularization parameter $\gamma$ over the PMI scores on 20NG, NYT and Twitter (L). Figure~\ref{fig:parameter-gamma-All} shows the PMI score obtained by NSMDM (Left) and NSMTM (Right) for $\gamma\in\{0.5, 0.75, 1, 1.25, 1.5\}$. We observe that NSMTM is more robust with respect to $\gamma$ than NSMDM and the smallest value of $\gamma$ leads to the best PMI score. This confirms our theoretical analysis in subsection~\ref{subsec:inference-network} that RW regularization is a good alternative to KL regularization with a proper choice of $\gamma$. 

Then we turn to explore the effect of learning rate $\eta$ over the PMI score on 20NG, NYT and Twitter (L). Figure~\ref{fig:parameter-eta-All} shows the PMI score obtained by NSMDM (Left) and NSMTM (Right) for $\eta\in 10^{-5}\times\{1, 5, 10, 50, 100, 500\}$. We observe that NSMTM is again more robust w.r.t. $\eta$ than NSMDM while both approaches may diverge for some large value of $\eta$. Also, the choice of $\eta$ is crucial for NSMDM: the range of $\left[0.0001, 0.001\right]$ works much better than other choices. 
\begin{figure}[t]
\includegraphics[width=0.23\textwidth]{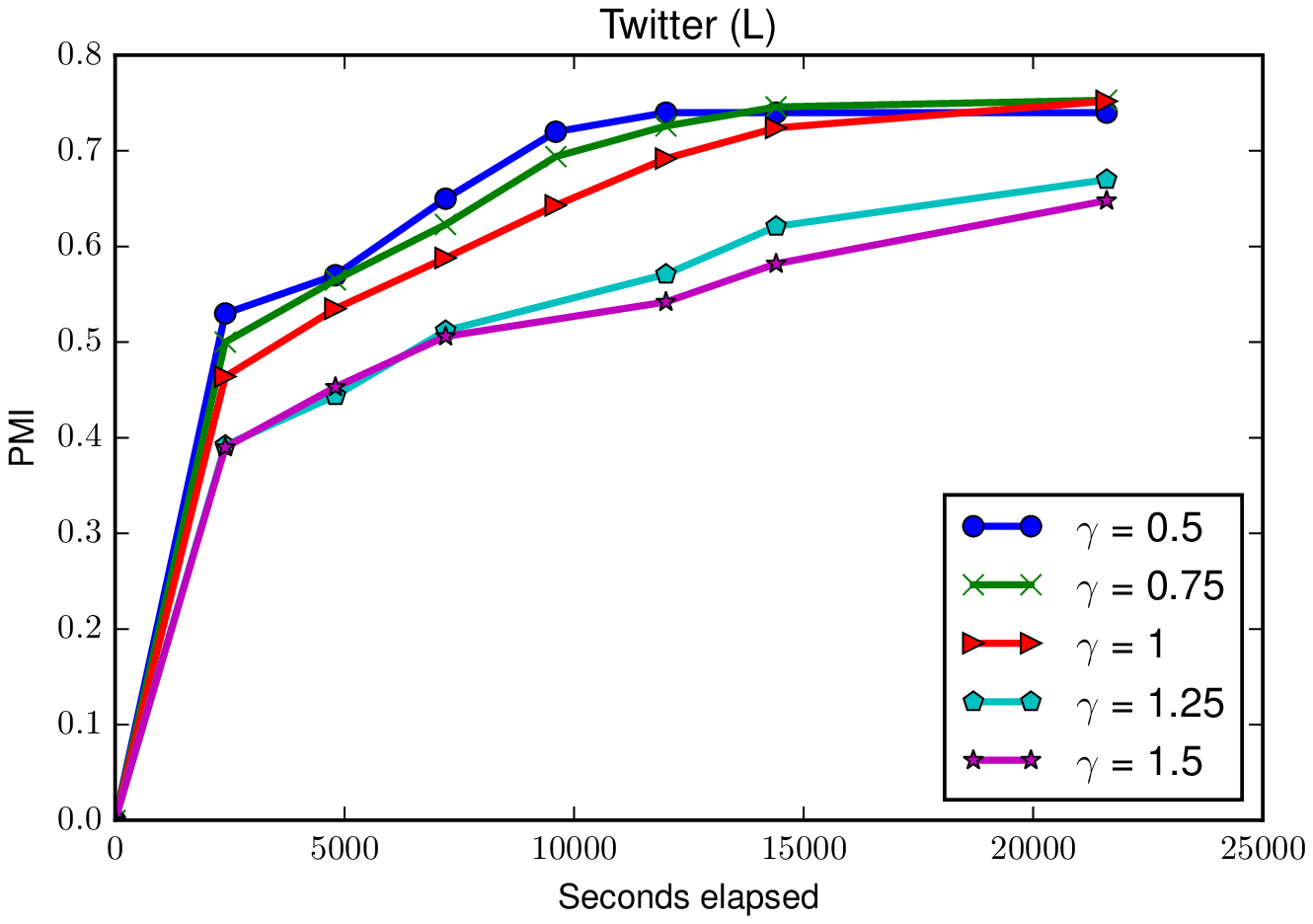}
\includegraphics[width=0.23\textwidth]{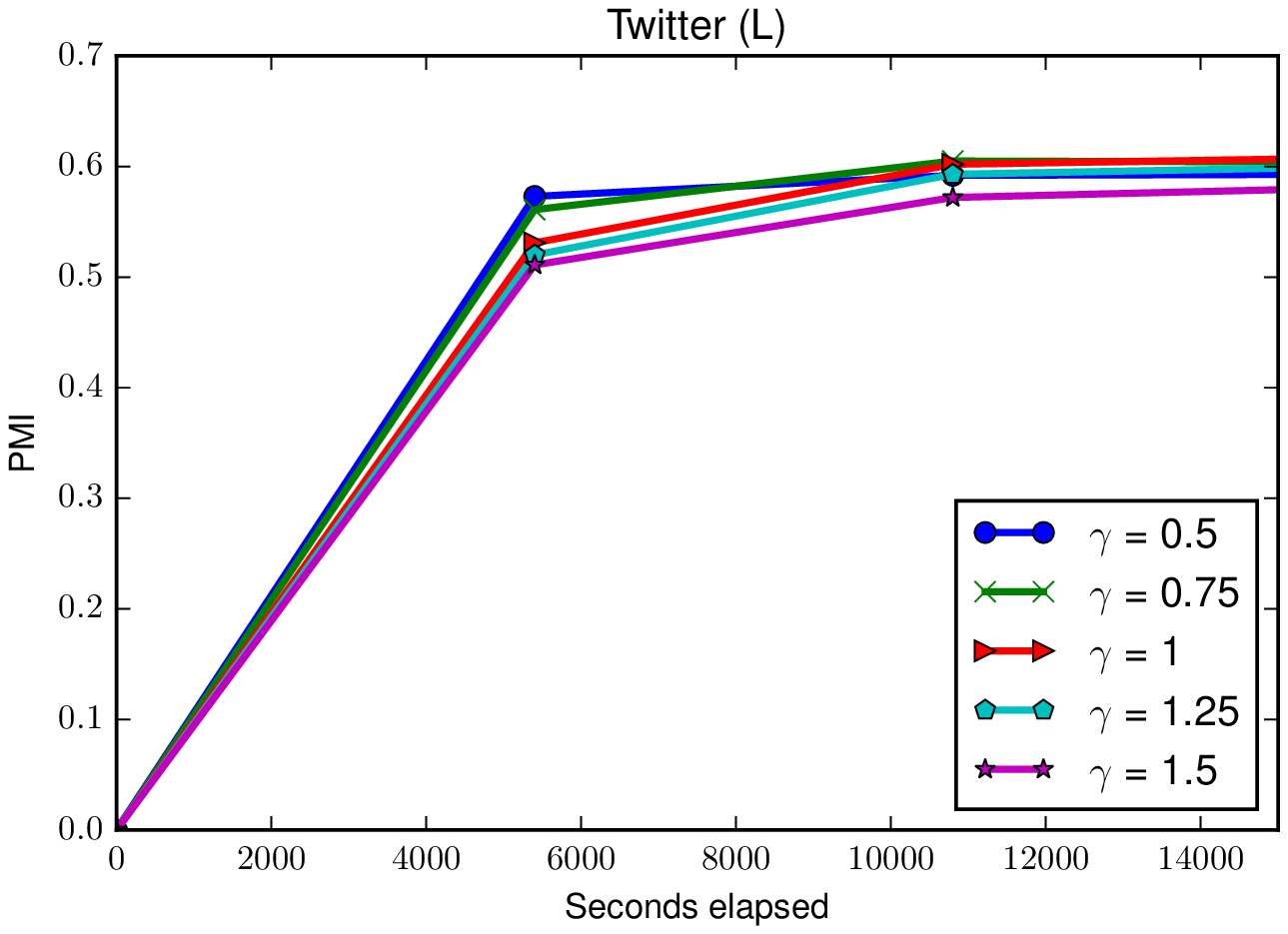}
\includegraphics[width=0.23\textwidth]{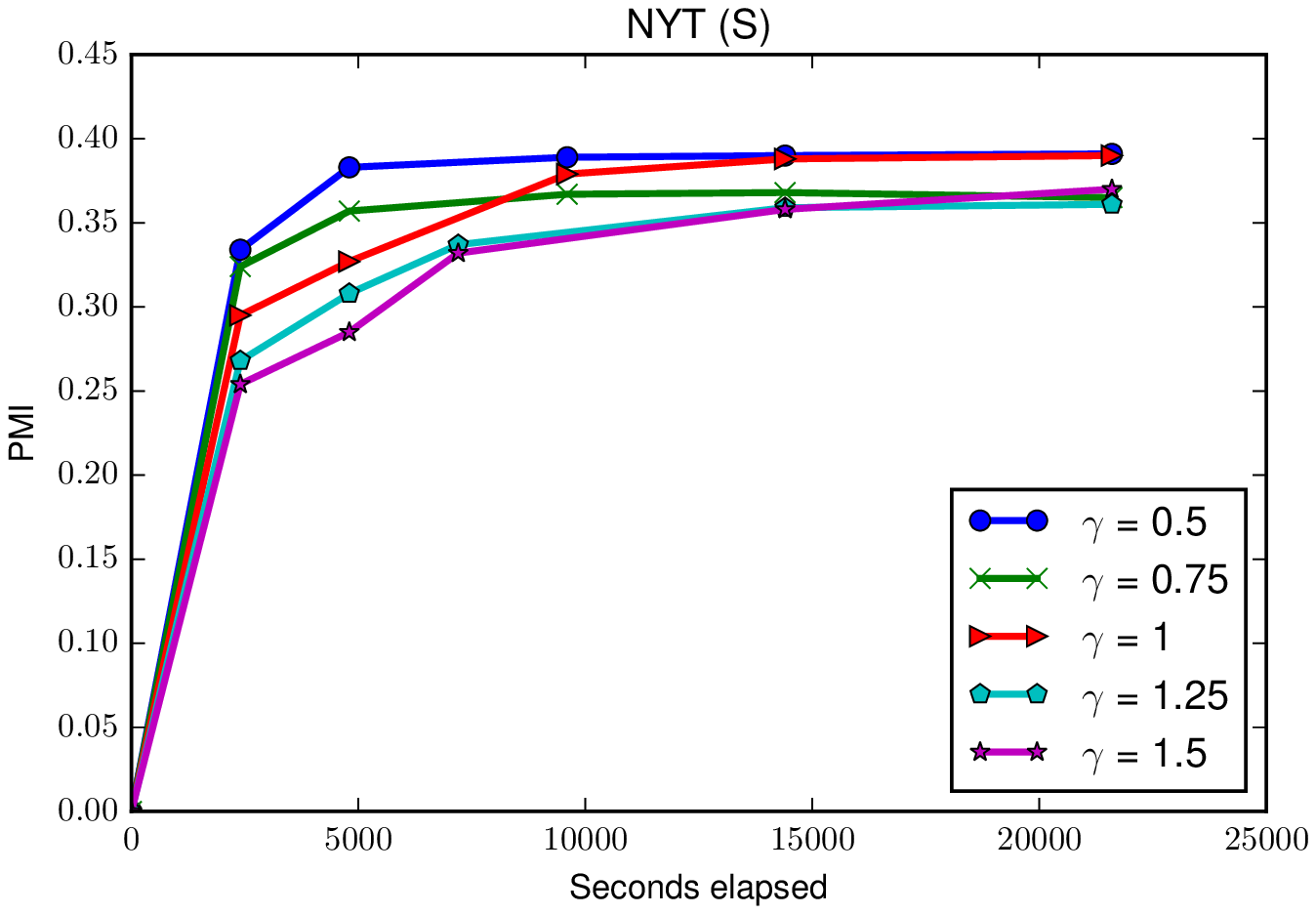}
\includegraphics[width=0.23\textwidth]{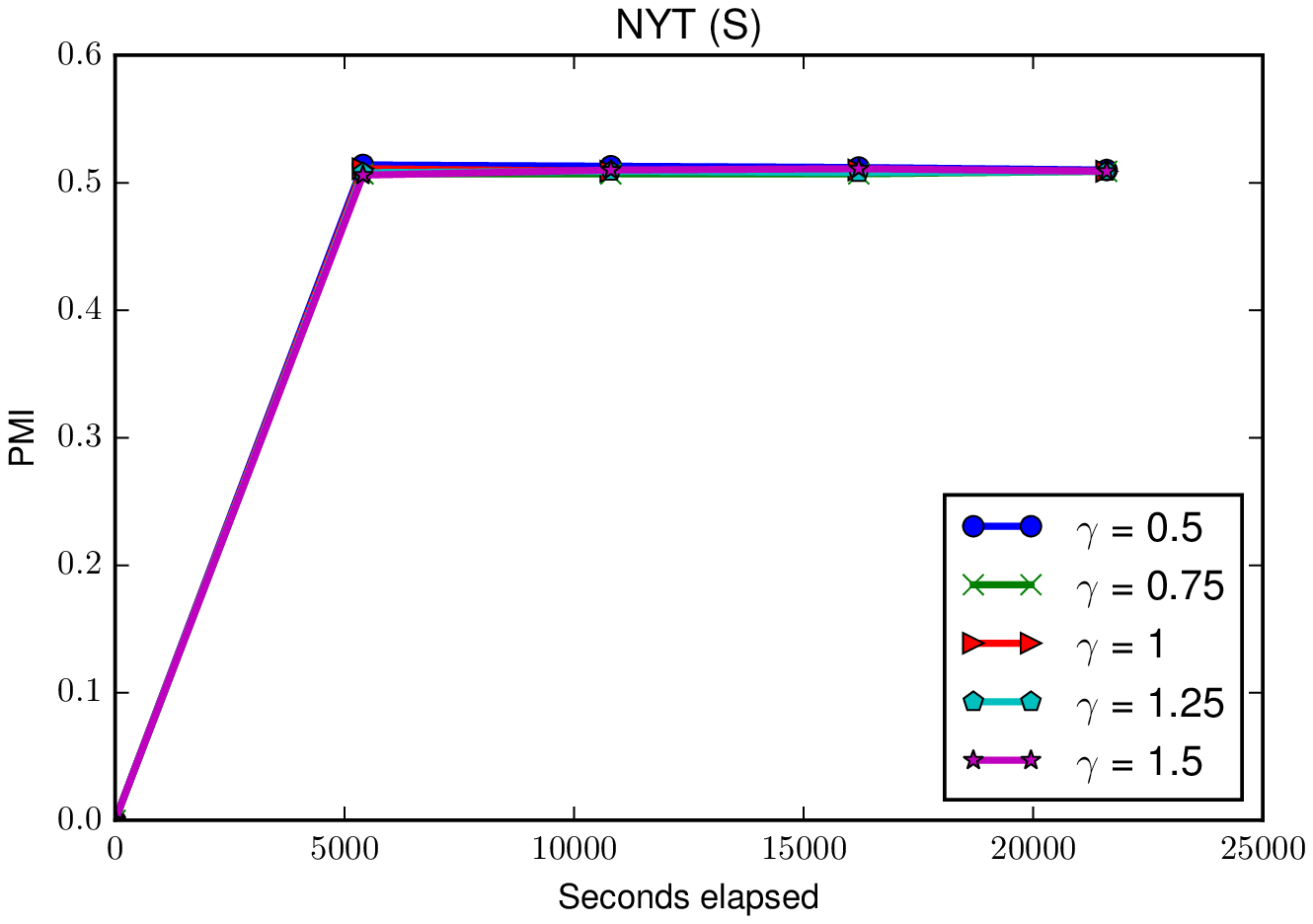}
\includegraphics[width=0.23\textwidth]{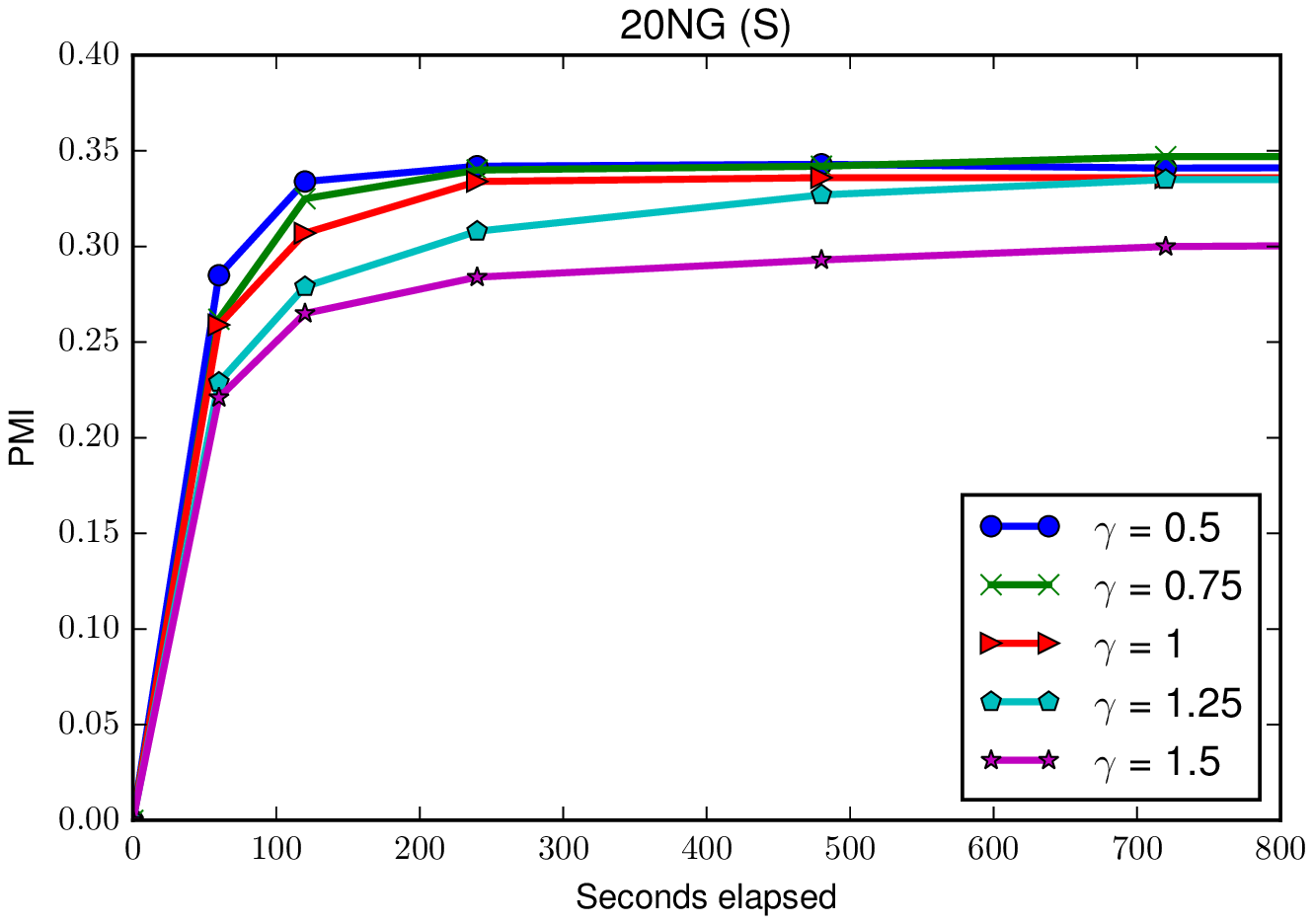}
\includegraphics[width=0.23\textwidth]{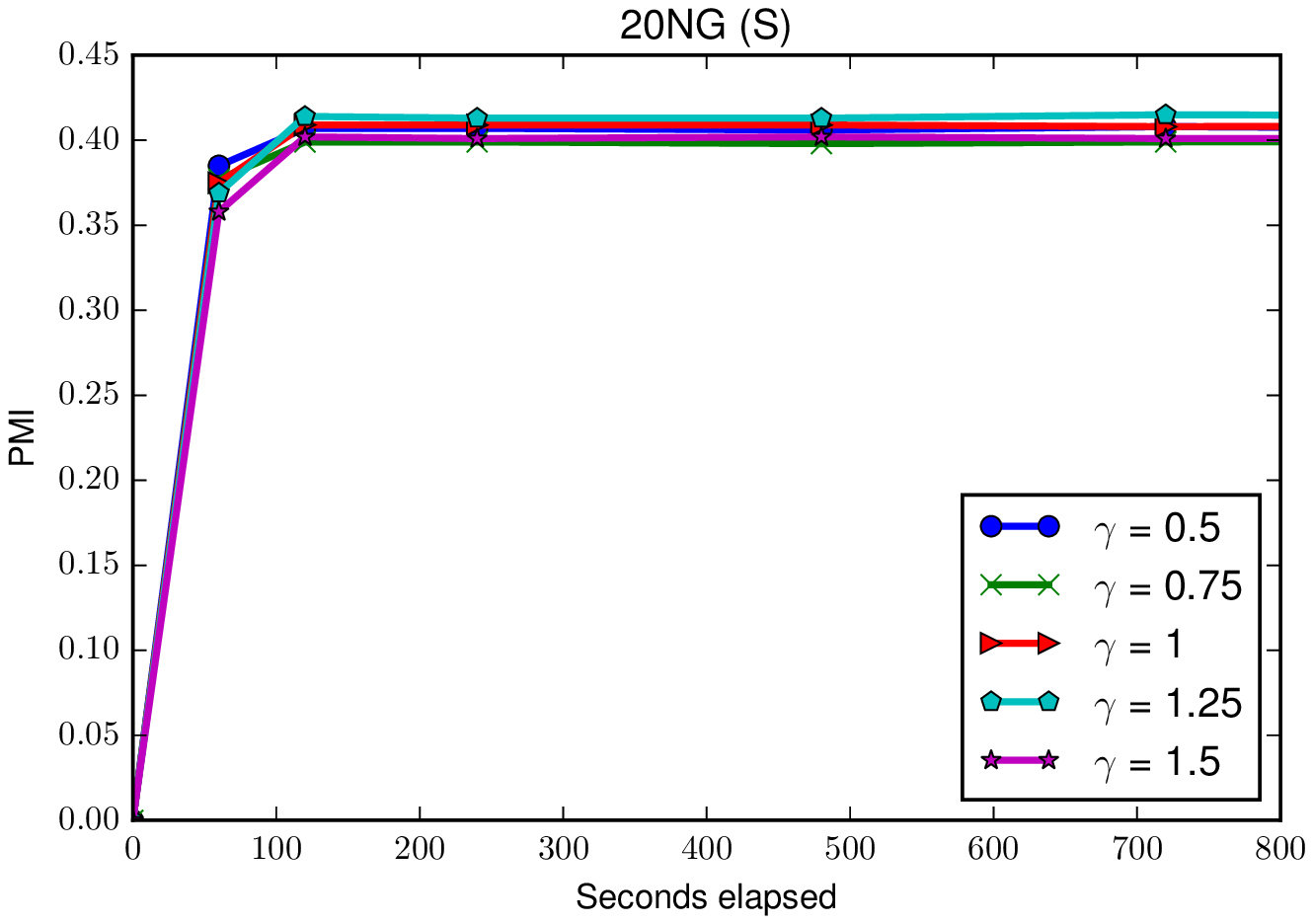}
\caption{Parameter sensitivity w.r.t $\gamma$ on 20NG, NYT and Twitter (L)}\label{fig:parameter-gamma-All}
\end{figure}
\begin{figure}[t]
\includegraphics[width=0.23\textwidth]{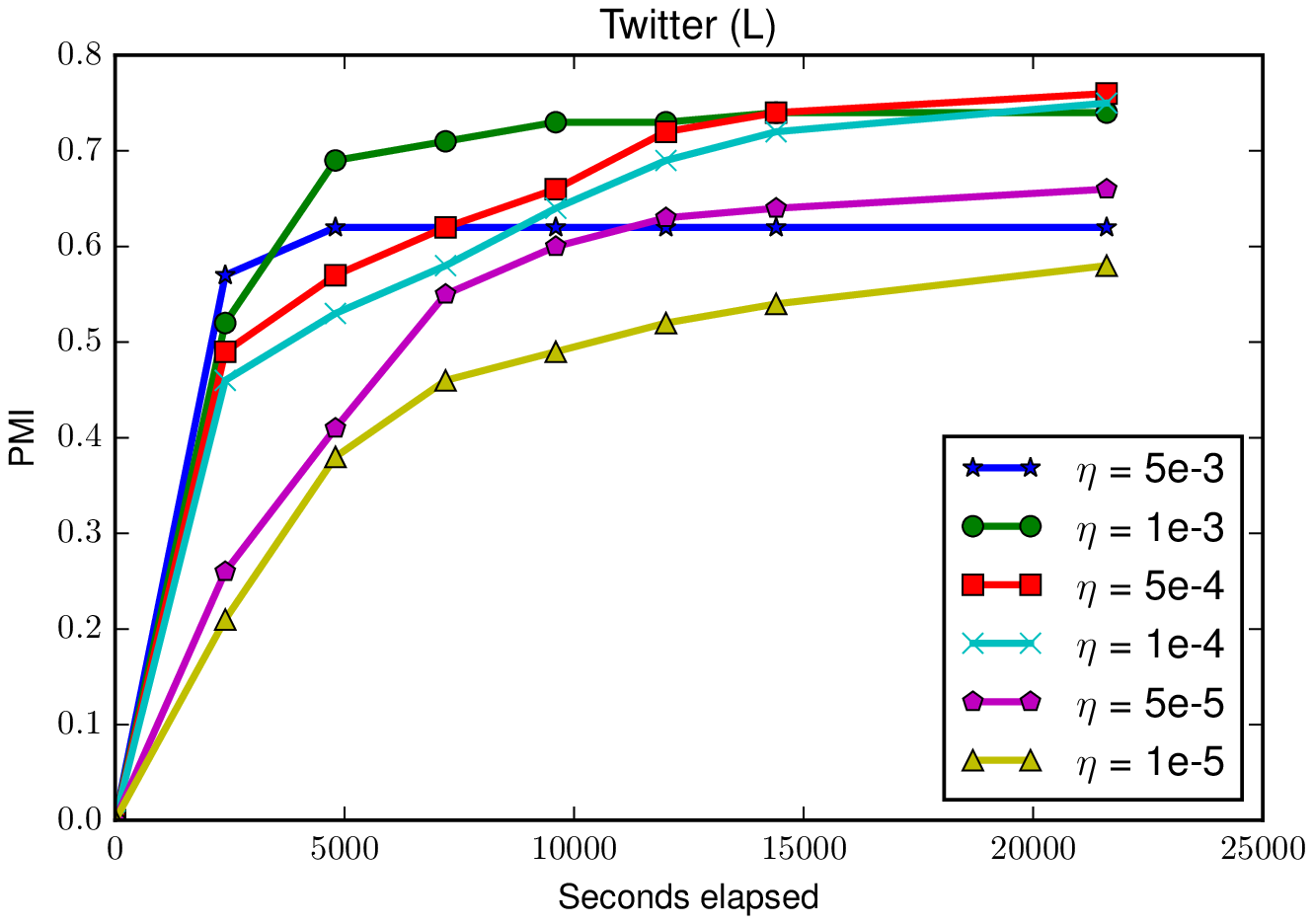}
\includegraphics[width=0.23\textwidth]{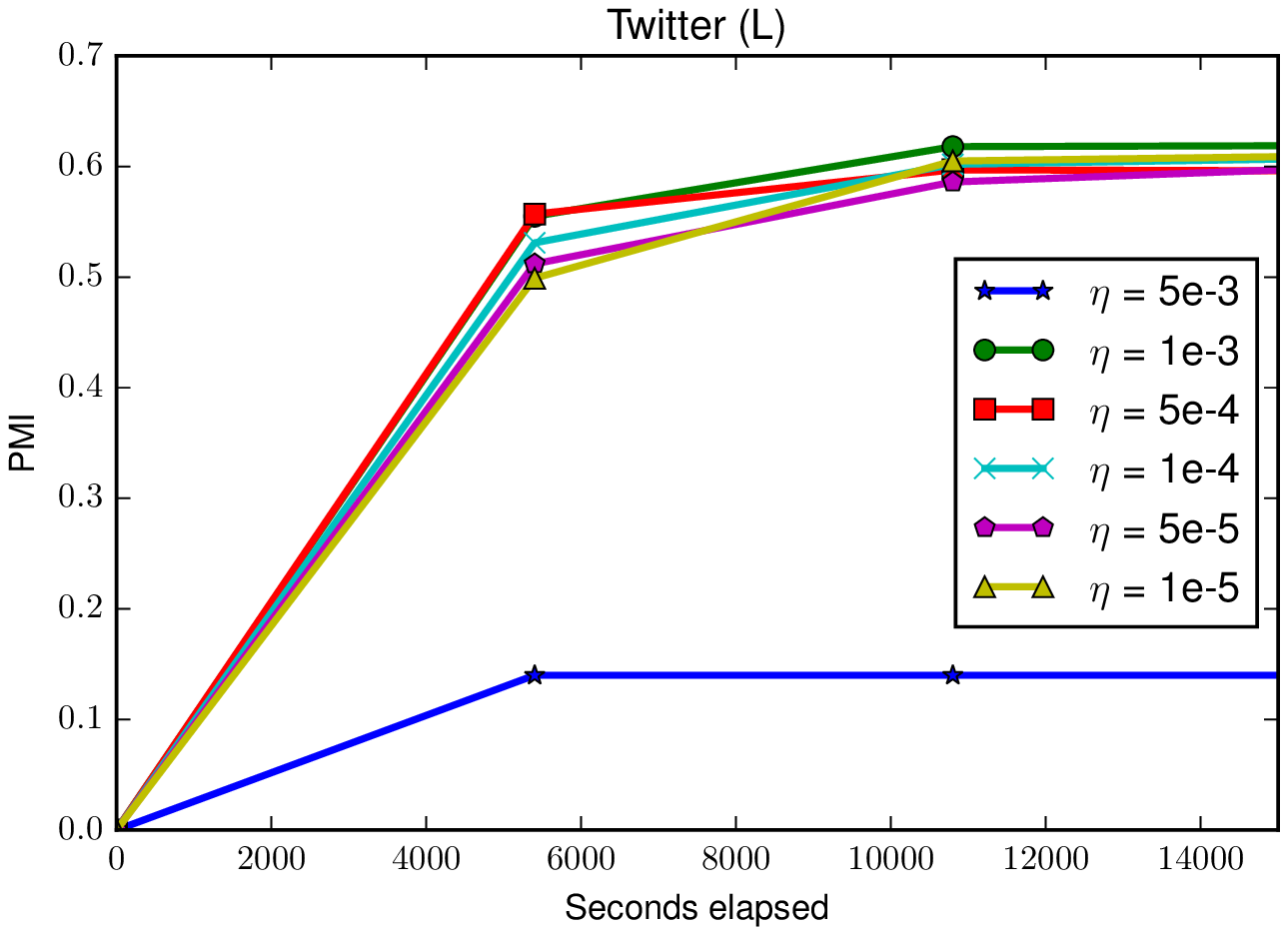}
\includegraphics[width=0.23\textwidth]{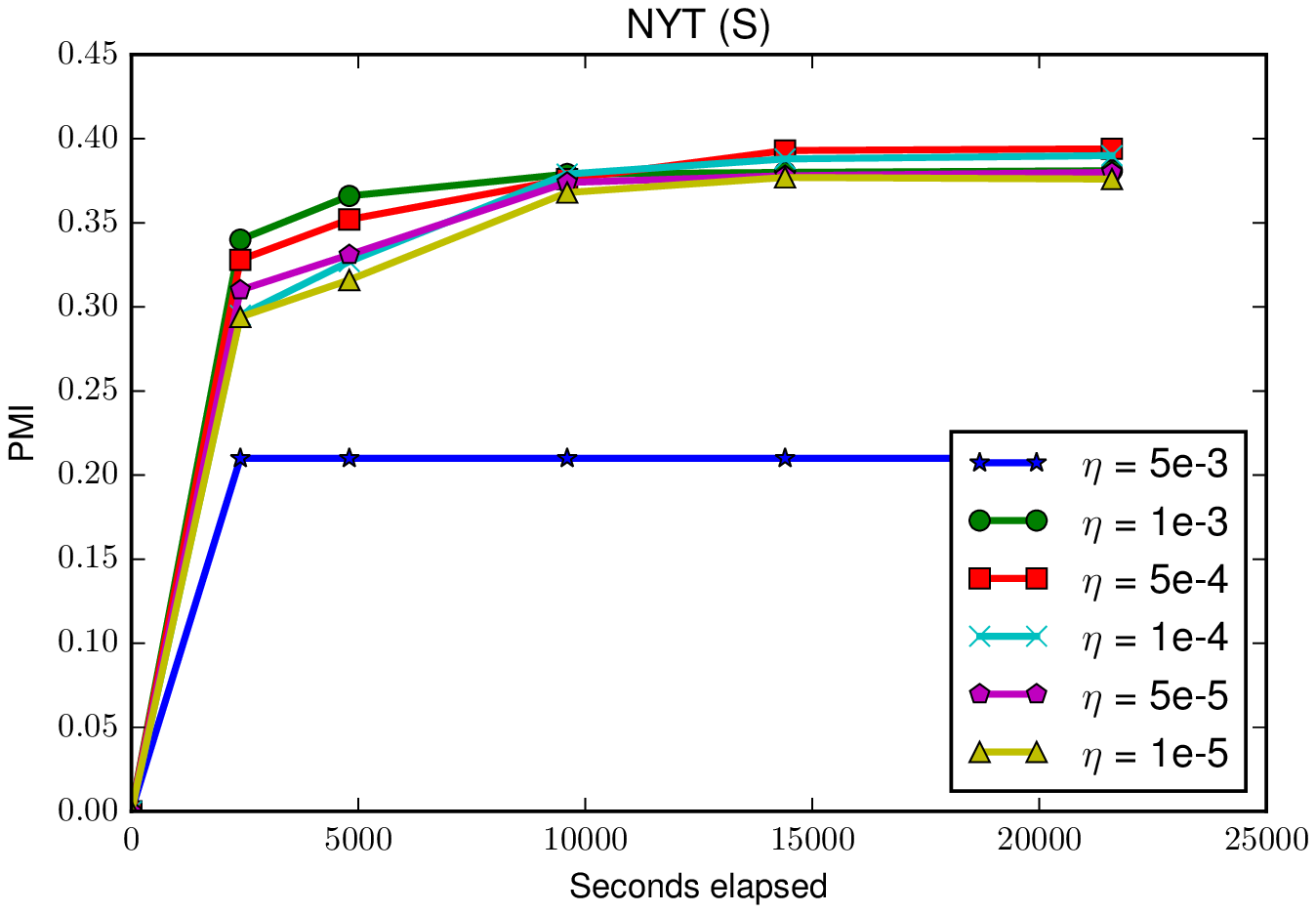}
\includegraphics[width=0.23\textwidth]{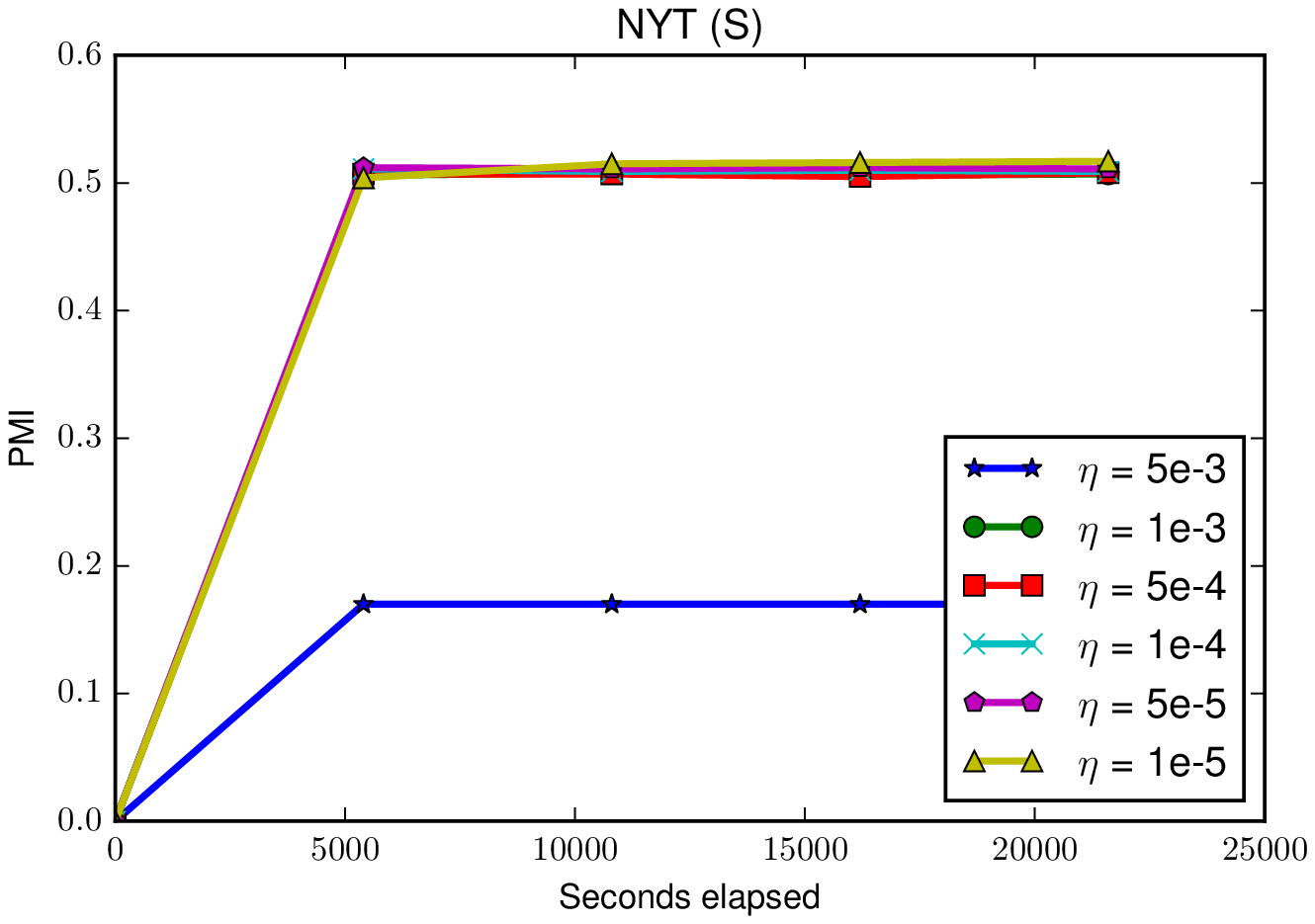}
\includegraphics[width=0.23\textwidth]{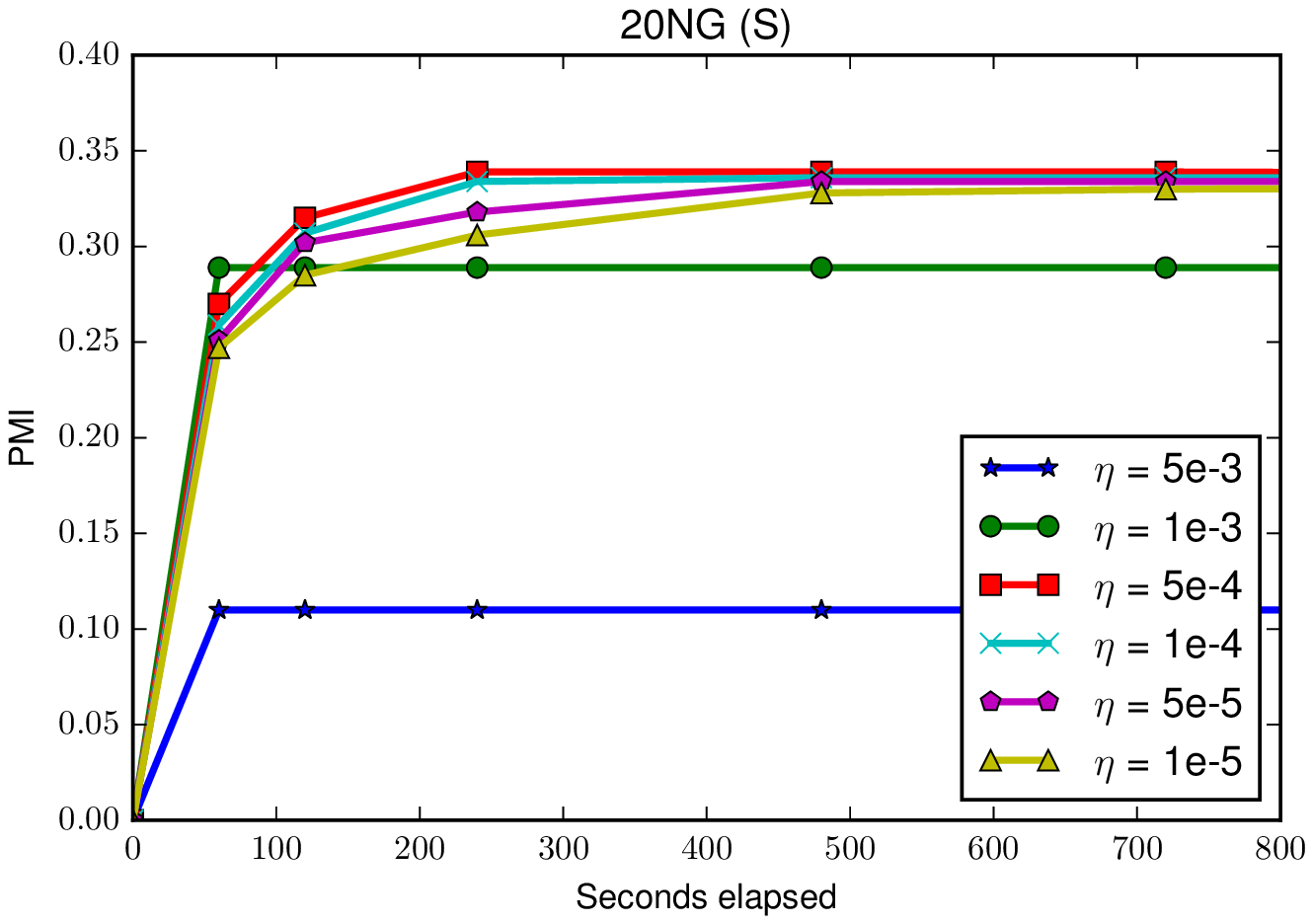}
\includegraphics[width=0.23\textwidth]{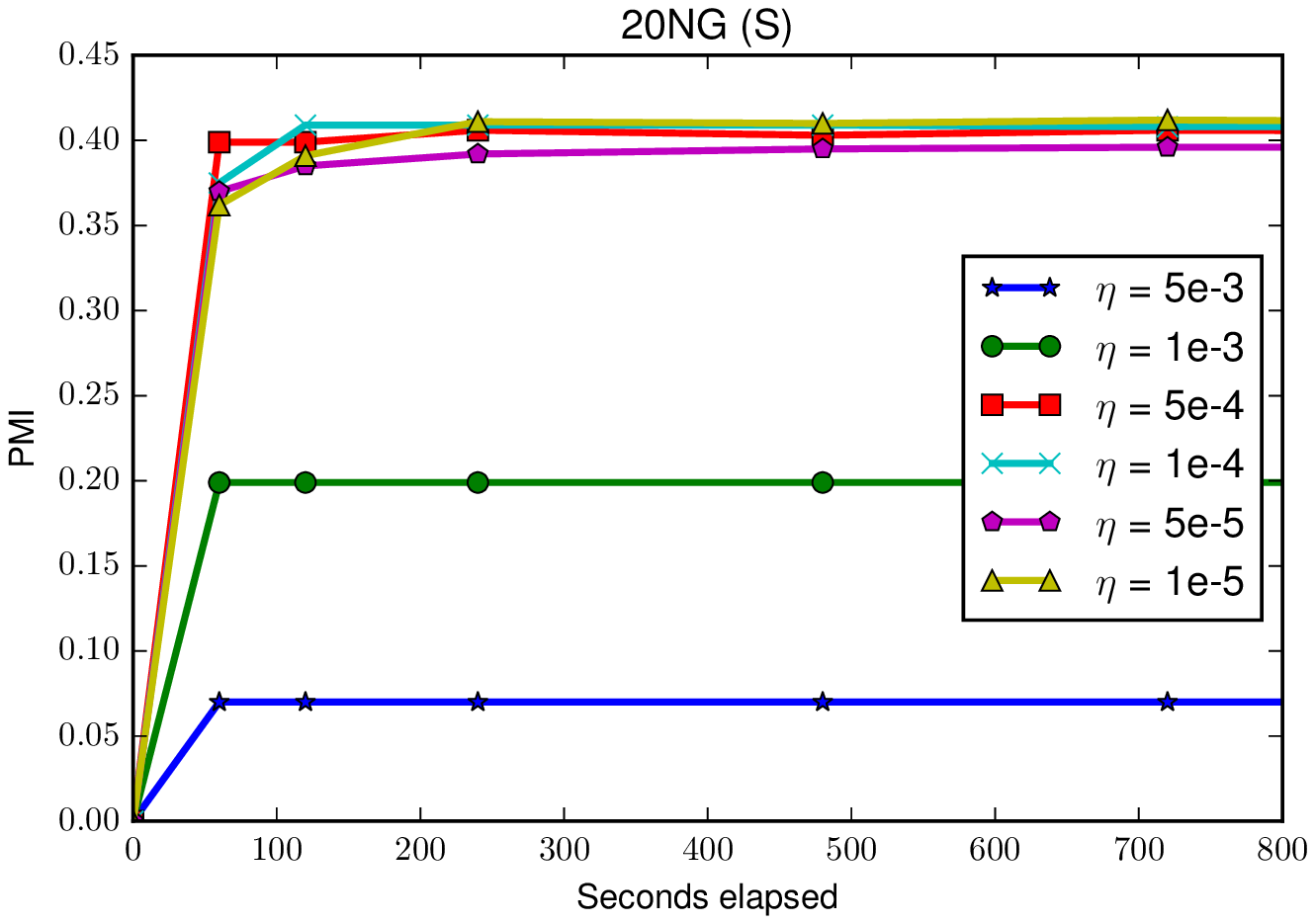}
\caption{Parameter sensitivity w.r.t $\eta$ on 20NG, NYT and Twitter (L)}\label{fig:parameter-eta-All}
\end{figure}

\subsubsection{Topic Interpretation} 
We present some selected topics on NYT in Table~\ref{tab:topic_interpretation_NSMTM} and~\ref{tab:topic_interpretation_NSMDM}. Both methods capture some interesting topics composed of words that are highly correlated. In Table~\ref{tab:topic_interpretation_NSMTM}, the first topic includes Fort Detrick, a US Army Medical Command installation, along with many biological terms and disease names, which consistently refer to biological contents. The second topic is centered around two telecommunication companies -- Lucent and Cisco -- whose ``networking war"  attracted public attention in early 2000. The third topic is reflective of music industry: Billboard is a popular music chart; Ravi Shankar is a famous Indian musician; arranger is a job to create a harmonic combination of different instrumental tracks for a song. In Table~\ref{tab:topic_interpretation_NSMDM}, the first one includes an extensive list of professional baseball terms. All words in the second topic are related to cooking, from ingredients, styles, tools to materials. The third one includes Napster, one of the earlier music streaming services online, UMG, one of the biggest copyright groups in the music industry, and mp3, a common music format  -- it clearly refers to the digital music streaming.
\begin{table}[t]
\caption{Selected Topics Inferred by NSMTM on NYT.} \label{tab:topic_interpretation_NSMTM}
\centering
\begin{tabular}{|c|c|c|} \hline
Biology & Telecomm War & Music Industry \\ \hline \hline
bacteria & stock & album \\
vaccine & Lucent & Billboard \\
germ & NASDAQ & saxophonist \\
bacterial & Cisco System & guitarist \\
cloning & euros & melodies \\
antibodies & DOW & Ravi Shankar \\
genes & analyst & San Francisco ballet \\
organism & capitalization & arranger \\ \hline
\end{tabular}
\end{table}
\begin{table}[t]
\centering
\caption{Selected Topics Inferred by NSMDM on NYT.} \label{tab:topic_interpretation_NSMDM}
\begin{tabular}{|c|c|c|c|c|c|} \hline
Baseball & Cooking & Digital Music Streaming \\ \hline \hline
playoff & tablespoon & user \\
league & teaspoon & Internet \\
baseman & garnish & Napster \\
pitcher & saucepan & consumer \\
season & cloves  & mp3 \\
coach & skillet & download  \\
homer & saute & Universal Music Group \\
defenseman & onion & aol \\ \hline
\end{tabular}
\end{table}

\section{Conclusion}\label{sec:conclusion}
In this paper, we propose two neural models NSMDM and NSMTM, and infer them on large text corpora through a novel inference procedure based on the RW divergence. The proposed approaches can discover the topic sparsity in very large short text corpora, performing better than all existing methods in terms of both the quality of solution and the stability of training. These simple yet effective generative and inference networks are feasible for training and testing on the GPU platform, and enhance the efficiency. 

Experimental results on different genres of large-scale text corpora demonstrate that the proposed approaches consistently achieve \textit{higher PMI score and lower perplexity} than other methods on large-scale collection of short text, and \textit{extract useful topics from about fifty million tweets within only 6 hours} while identifying sparsity in the topical proportion of each tweet. Due to their simplicity and ease-of-implementation, we hope that NSMDM and NSMTM may be helpful for analyzing huge volume of short text which becomes prevalence in the era of social media.



\end{document}